\documentclass[%
aip,
amsmath,amssymb,
reprint,%
]{revtex4-1}

\usepackage{graphicx}% Include figure files
\usepackage{dcolumn}% Align table columns on decimal point
\usepackage{bm}% bold math
\usepackage[utf8]{inputenc}
\usepackage[T1]{fontenc}
\usepackage{mathptmx}
\usepackage{subcaption}
\usepackage{color}
\usepackage[normalem]{ulem}

\usepackage{soul,xcolor}

\usepackage{amsmath,amssymb,amsthm}
\newtheoremstyle{query}%
{}{}%space above/below
{}%body style
{}%heading indent
{\sffamily\bfseries}{:}{12pt}%heading style/punctuation/space after
{}% head spec
\theoremstyle{query}

\usepackage{graphicx} % Required for inserting images
\usepackage[]{babel}
\usepackage[]{amsfonts} \usepackage[]{amsmath}
\usepackage[]{amssymb} \usepackage[]{latexsym}
\usepackage{graphicx,color} \usepackage{amsthm}
\usepackage{fancyhdr} \usepackage{url}
\usepackage{amssymb} \usepackage{epigraph}
\usepackage{color}   \usepackage{bbm}
\usepackage{caption} 
\usepackage{pifont} \usepackage{enumerate}
\usepackage{mathrsfs} \usepackage{subcaption}
\usepackage{pythonhighlight}
\allowdisplaybreaks

\newcommand{\R}{{\mathbb R}}

\newtheorem{theorem}{Theorem}[section] \newtheorem*{theorem*}{Theorem}
\newtheorem*{nonameteo*}{}
 \newtheorem*{proposition*}{Proposition}
\newtheorem{lemma}[theorem]{Lemma} \newtheorem*{lemma*}{Lemma}
\newtheorem{corollary}[theorem]{Corollary} \newtheorem*{corollary*}{Corollary}
 \newtheorem*{conjecture*}{Conjecture} \newtheorem*{example*}{Example}

\begin{document}

\setstcolor{red}

\preprint{AIP/123-QED}

\title[Similarity Learning with neural networks]{Similarity Learning with neural networks}
% Force line breaks with \\

\author{G. Sanfins}
\email{gabrielsanfins@matematica.ufrj.br}
\affiliation{ 
Department of Applied Mathematics, Institute of Mathematics, Federal University of Rio de Janeiro, Centro de Tecnologia, Bloco C, Av. Athos da Silveira Ramos - Cidade Universitária, Rio de Janeiro, RJ, 21941-909, Brazil}%

\author{F. Ramos}
\email{framos@ufrj.br.}
\affiliation{ 
Department of Applied Mathematics, Institute of Mathematics, Federal University of Rio de Janeiro, Centro de Tecnologia, Bloco C, Av. Athos da Silveira Ramos - Cidade Universitária, Rio de Janeiro, RJ, 21941-909, Brazil}%

\author{D. Naiff}
\email{dfnaiff@mecanica.coppe.ufrj.br}
\affiliation{ 
Department of Mechanical Engineering, Federal University of Rio de Janeiro, Centro de Tecnologia, Bloco C, Av. Athos da Silveira Ramos - Cidade Universitária, Rio de Janeiro, RJ, 21941-909, Brazil}%

\date{\today}% It is always \today, today,
% but any date may be explicitly specified

\begin{abstract}
In this work, we introduce a neural network algorithm designed to automatically identify similarity relations from data. By uncovering these similarity relations, our network approximates the underlying physical laws that relate dimensionless quantities to their dimensionless variables and coefficients. Additionally, we develop a linear algebra framework, accompanied by code, to derive the symmetry groups associated with these similarity relations. While our approach is general, we illustrate its application through examples in fluid mechanics, including laminar Newtonian and non-Newtonian flows in smooth pipes, as well as turbulent flows in both smooth and rough pipes. Such examples are chosen to highlight the framework's capability to handle both simple and intricate cases, and further validates its effectiveness in discovering underlying physical laws from data.\end{abstract}

\maketitle

\section{Introduction} \label{sec:intro}

Understanding and predicting the behavior of complex physical systems is a cornerstone of scientific and engineering endeavors. In fluid mechanics, for instance, accurately simulating real operational conditions is essential for the design and optimization of pipelines, aerospace components, and various industrial processes. However, full-scale simulations of such systems are often prohibitively expensive and time-consuming due to the intricate dynamics and vast parameter spaces involved. This poses a significant challenge for researchers and engineers who seek to explore and optimize these systems efficiently.

One promising approach to mitigate these challenges is the identification of scaling similarities and symmetry groups within physical systems. By uncovering the correct scaling relations, we can develop smaller, more manageable models that accurately capture the essential behavior of real-world scenarios. These scaled models not only reduce computational costs but also accelerate the design and testing processes by allowing for efficient exploration of the parameter space. Moreover, understanding these scaling laws deepens our insight into the fundamental principles governing these systems, enabling us to generalize findings from simplified models to full-scale applications with greater confidence.

In recent years, the application of machine learning in fluid mechanics has been on the rise, offering innovative tools to address complex problems that are difficult to solve analytically. Machine learning techniques have been employed to model turbulent flows, optimize fluid systems, and discover new physical laws from data \cite{Brunton2020, Vinuesa2023, Kochkova2021}. In particular, Bakarji et al.\cite{bakarji2022} recently developed a series machine learning frameworks, including a cleverly designed neural network named Buckinet, in order to discover dimensionless representations from data based on Buckingham's $\Pi$ theorem. These advancements underscore the potential of data-driven approaches in enhancing our understanding and predictive capabilities in fluid dynamics. However, integrating machine learning with fundamental principles like scaling similarities and symmetry groups remains a challenging yet promising avenue for research.

Let us illustrate the importance of obtaining similarity laws and symmetry groups with the task of controlling the friction factor \( f \) in Newtonian fluid flows by adjusting the average flow velocity \( \bar{U} \) across a pipe's cross-section. The friction factor is defined as
\begin{equation} \label{eq:friction_def}
f = 2 \frac{u_{\tau}^2}{\bar{U}^2} = \frac{2\tau_w}{\rho \bar{U}^2} = \frac{D\left(-\frac{\partial P}{\partial z}\right)}{2\rho \bar{U}^2},
\end{equation}
where \( \frac{\partial P}{\partial z} \) denotes the pressure drop per unit length, \( D \) is the pipe diameter, \( \tau_w \) is the wall shear stress, \( \rho \) is the fluid density and $u_\tau:=\sqrt{\tau_w/\rho}$ is a dimensionless quantity known as the friction velocity. 

A straightforward application of dimensional analysis reveals that, for Newtonian flows, this dimensionless quantity depends solely on the bulk Reynolds number \( Re = \frac{\rho \bar{U} D}{\mu} \), where \( \mu \) is the fluid's dynamic viscosity.

In laminar flows, the friction factor follows the relationship \( f \sim 1/Re \). Consequently, scaling the average flow velocity by a factor \( A_1 \), such that \( \bar{U}^* = A_1 \bar{U} \), results in the friction factor scaling as \( f^* = A_1^{-1} f \).

For turbulent flows, within the Reynolds number range \( 4 \times 10^3 \lesssim Re \lesssim 10^5 \), the Blasius correlation provides a useful approximation: \( f \sim 1/Re^{1/4} \). In this case, scaling the flow velocity by \( B_1 \), such that \( \bar{U}^* = B_1 \bar{U} \), leads to the friction factor scaling as \( f^* = B_1^{-1/4} f \).

The distinct scaling laws for laminar and turbulent regimes indicate that these phenomena cannot be derived solely through simple dimensional analysis, and, moreover, that similarity and scaling may vary depending on the range of dimensionless parameters. For laminar flows, the power-law relationship can be rigorously derived from the Navier-Stokes equations, allowing for an explicit solution linking the friction factor to the flow rate. In contrast, the $1/4$- scaling for turbulent flows was obtained semi-empirically by Blasius in 1911 based on a groundbreaking series of experiments. In 2006, Gioia and Chakraborty developed sophisticated models explaining Nikuradse's roughness data\cite{Nikuradse}, which included a semi-empirical derivation of Blasius' scaling.

Also in 2006, Goldenfeld \cite{Goldenfeld2006} presented empirical evidence suggesting that turbulent flows in rough pipes share certain characteristics with critical phenomena, based on friction factor measurements. These flows, for instance, exhibit a broad range of length scales and follow power-law behavior with scale-invariant correlated fluctuations. Another similar aspect of critical phenomena, revealed by Goldenfeld for turbulent flows in rough pipes, is data collapse, or Widom scaling \cite{Lectures_Goldenfeld}, where a physical relation that initially depends on two dimensionless variables can be reduced to a single dimensionless quantity.

More specifically, Goldenfeld's analysis builds on the pioneering work of Nikuradse \cite{Nikuradse}, who, in 1932 and 1933, conducted a landmark series of measurements on flow in both nominally smooth and rough pipes.
 In his experiments on sand-roughened pipes, Nikuradse used sand grains of well-defined sizes \( r \), covering a wide range of values, and pipes with different diameters \( D \). He confirmed the expectation of hydrodynamic similarity: the flow properties depend on the roughness only through the ratio \( r/D \). He presented results for the shear force per unit area \( \tau \) exerted by the flow on the walls of the pipes. These data are shown in Fig. \ref{fig:nikuradse_data}.

\begin{figure*}[htp]
    \centering
    \includegraphics[width=17cm]{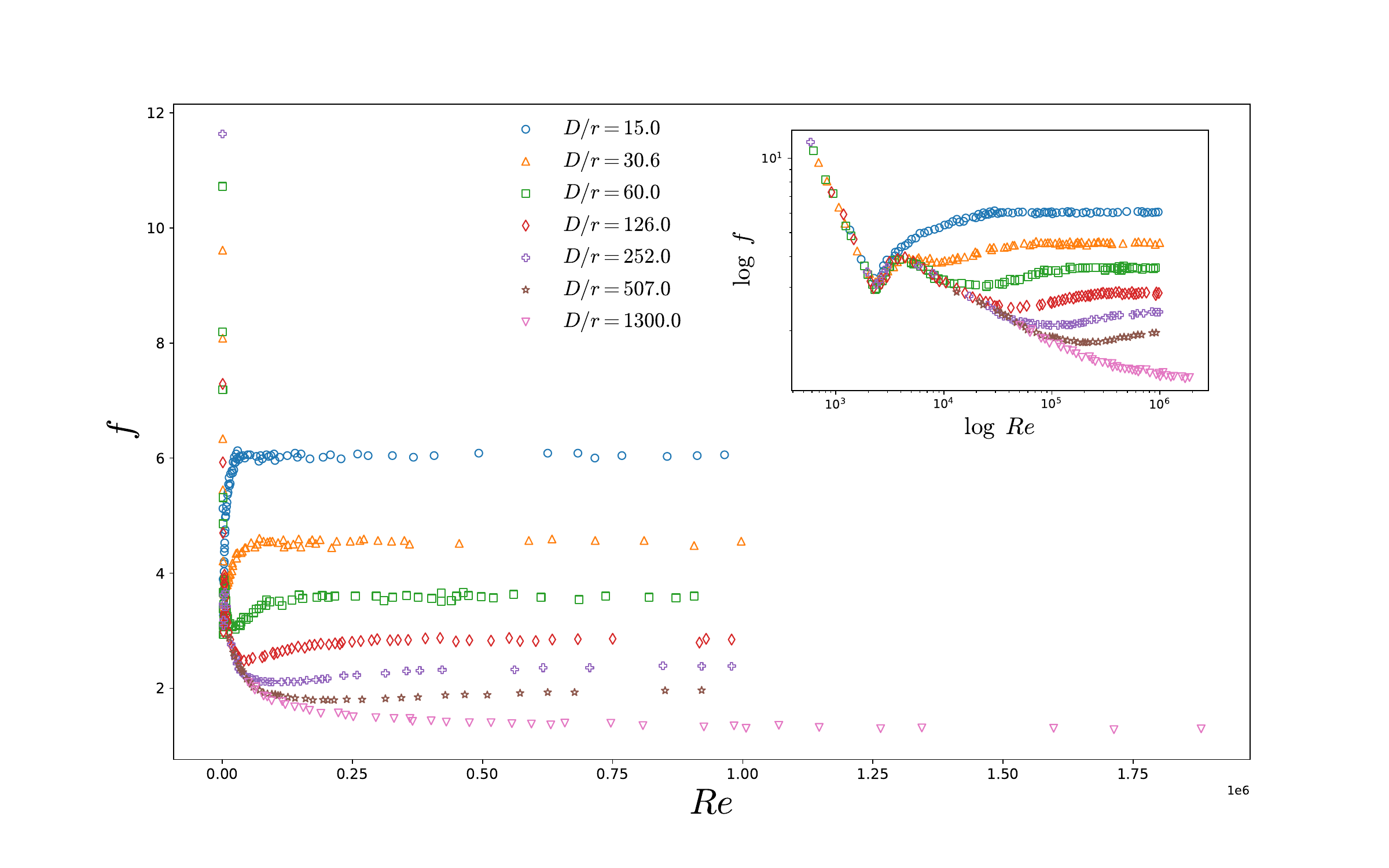}
    \caption{Nikuradse's roughness data.}
    \label{fig:nikuradse_data}
\end{figure*}

\begin{figure*}[htp]
    \centering
    \includegraphics[width=17cm]{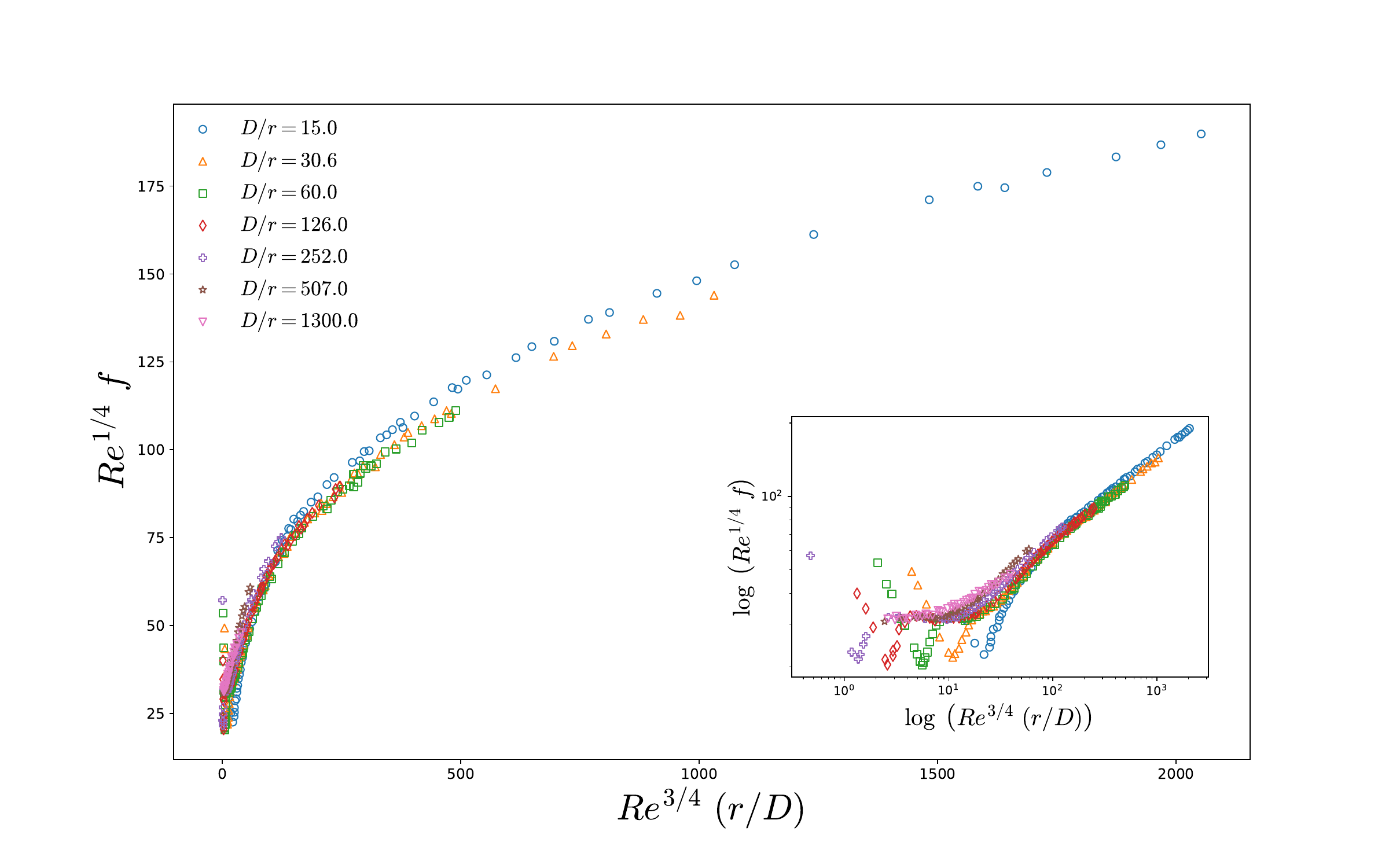}
    \caption{Nikuradse's roughness data collapsed into a single curve with similarity exponents proposed by Goldenfeld.}
    \label{fig:nikuradse_data_goldenfeld_exponents}
\end{figure*}

By analyzing Nikuradse's data and drawing from Blasius' scaling for smooth pipes and Strickler's law, which states that the friction factor \( f \) for turbulent flow in a rough pipe scales as \( f \propto \left( \frac{r}{D} \right)^{1/3} \), Goldenfeld demonstrated a data collapse onto a single curve when plotted with the appropriate scaling. This analogy to Widom scaling in critical phenomena predicts that the friction factor data, when plotted as \( f \times \text{Re}^{1/4} \) versus \( \left( \frac{r}{D} \right) \text{Re}^{3/4} \), collapse onto a single curve, as seen in Fig. \ref{fig:nikuradse_data_goldenfeld_exponents}. This highlights the similarity to the data collapse observed in critical systems.

Goldenfeld's results have significant practical implications for experimental design. They establish an invariant symmetry group that relates the friction factor to the flow rate in rough pipes. Unlike in laminar flows or turbulent flows within the range of validity of Blasius' correlation, where a direct connection between scaling the flow rate and scaling the friction factor can be made, the scaling in rough pipes requires a more nuanced approach. Nevertheless, thanks to Goldenfeld's findings, the following scaling relationship holds:

\begin{equation} \label{eq:goldenfeld_renormalization_group}
    \bar{U}^* = B_1 \bar{U}; \quad r^* = B_1^{-3/4} r \quad \Rightarrow \quad f^* = B_1^{-1/4} f.
\end{equation}

Simulating such complex fluid systems under real operational conditions remains costly and resource-intensive. High-fidelity computational models and large-scale experiments demand substantial computational power and time, making it impractical to explore the vast parameter spaces involved. While identifying and exploiting the correct scaling similarities allows us to develop smaller, more manageable models that accurately represent real-world applications, obtaining such relations in general physical systems is challenging. The analytical examples discussed earlier are exceptions rather than the norm; in most cases, deriving these scaling relations analytically is difficult due to the complexity of the systems. This underscores the need for methods capable of uncovering scaling similarities and symmetry groups directly from data.

Although Goldenfeld drew parallels between rough pipe flow and critical phenomena, renormalization relations and reduced dimensionless variables are common features in many physical systems, even in the absence of a wide range of length scales \cite{barenblatt_2003}. For example, depending on the choice of governing parameters, even laminar flows can exhibit similar scaling behaviors, as we demonstrate later in this work.

In fact, when a dimensionless quantity, such as \(\Pi_k\), is either very large or very small for a particular physical system, two distinct types of dimensionality reduction in the governing function can occur. On the one hand, we may be able to disregard the effects of \(\Pi_k\), leading to what is known as complete similarity. However, renormalization exponents may exist, leading to data collapse, similar to that found by Goldenfeld in the rough-pipe setting, which is termed incomplete similarity. In either case, the collapse is related to the presence of hidden similarity groups, as shown in Equation \eqref{eq:goldenfeld_renormalization_group}. These groups differ from those associated with dimensional analysis and Buckingham's \(\Pi\)-Theorem, as discussed in later sections of this manuscript.

In this work, we introduce an algorithm to automatically identify similarity relations from data using a simple neural network architecture which we call Barenet. We chose this name as an homage to the great scientist Grigory I. Barenblatt, whose works in similarity have greatly inspired our own findings. Additionally, we develop a linear algebra framework, together with a companion code, to derive the symmetry groups related to the similarity relations associated with a given problem. We illustrate the results with applications focused on fluid mechanics. For example, our network can approximate Goldenfeld's scaling laws automatically from Nikuradse's data, and our linear algebra module can automatically derive the symmetry relations as seen in \eqref{eq:goldenfeld_renormalization_group}.

More specifically, we begin with a brief review of dimensional analysis, the construction of dimensionless quantities, and Buckingham's $\Pi$-Theorem, as discussed in Barenblatt's book\cite{barenblatt_2003}. This is followed by more precise definitions of complete and incomplete similarity in Section \ref{sec:beyond_buckingham}. In Section \ref{sec:laminar_newtonian}, we provide a simple example of complete and incomplete similarity in the context of laminar Newtonian flows, demonstrating how the type of similarity in a phenomenon may depend on the chosen governing parameters.

In Section \ref{sec:generalized_similarity_groups}, we extend the concept of dimensionless construction to a more general setting and provide mathematical results to calculate similarity groups associated with this new construction theory. This approach is designed to include dimensionless quantities commonly used in fluid mechanics literature, such as the dimensionless distance to the wall, \(y^+\). 

In Section \ref{sec:Barenet}, we present a computational framework implemented in a Python package, which performs two main functions. First, it includes an automatic similarity group calculator based on the theorems discussed in Section \ref{sec:generalized_similarity_groups}. Second, it introduces a neural network architecture designed to automatically identify incomplete similarity exponents. We illustrate the application of this neural network to both laminar and turbulent flows in smooth and rough pipes. Companion codes are also provided for the similarity neural network and the similarity group analysis.

Section \ref{sec:HB_flows} delves into a novel and intricate example of incomplete similarity and renormalization group theory in Herschel-Bulkley laminar flows. This example is significant because it demonstrates a reduction in dimensionality from three dimensionless parameters to two through a non-trivial incomplete similarity relationship. The section concludes with an implementation of our algorithm for this non-trivial example.
\section{Dimensional Analysis} \label{sec:dimensional_analysis}

In any physical study, our goal is to establish relationships among the quantities involved in the phenomenon under investigation. Consequently, the problem can always be seen as determining one or several relationships of the form:
\begin{equation}\label{eq:general_physical_relation}
	a = F\left( a_1, \dots, a_m, b_1, \dots, b_l \right),
\end{equation}
where $a$ is the quantity of interest, or, in other words, the quantity we want to predict in such study. Its arguments are assumed to be given and are called \textbf{governing parameters}. We split the parameters so that $a_1, \dots, a_m$ are dimensionally independent, while $b_1, \dots b_l$ are dimensionally dependent on $a_1, \dots a_m$. In other words, we can choose exponents $\alpha_i^{(j)}$ with $1\leq i \leq m$ and $ 1 \leq j \leq l$ such that we can construct dimensionless quantities $\Pi_1, \dots, \Pi_l$ in the following way:
\begin{align} \label{eq:barenblatt_construction}
	\Pi_j &= b_j \cdot a_1^{\alpha_1^{(j)}} \cdots a_m^{\alpha_m^{(j)}}, \ \ \ \ \ \  j=1,\dots,l.
\end{align}
This implies that:
\begin{align}
	[b_j] &= [a_1]^{-\alpha_1^{(j)}} \cdots [a_m]^{-\alpha_m^{(j)}}, \ \ \ \ \ \  j=1,\dots,l,
\end{align}
where $[.]$ is the dimension function.

The dimension of the quantity of interest $a$  must also be expressible in terms of the dimensions of the governing parameters in the first group, $a_1, \dots, a_m$. By finding appropriate exponents $\alpha_1, \dots, \alpha_m$ to construct a dimensionless quantity $\Pi$ as follows:
\begin{equation} \label{eq:dimensionless_qoi_construction}
	\Pi = a \cdot a_1^{\alpha_1} \cdots a_m^{\alpha_m},
\end{equation}
which leads to the dimensional relationship
\begin{equation}
	[a] = [a_1]^{-\alpha_1} \cdots [a_m]^{-\alpha_m}.
\end{equation}

\begin{nonameteo*}[\textbf{Buckingham's $\Pi$-theorem}]
	Every physical relation, as expressed in \eqref{eq:general_physical_relation}, can be reduced to a dimensionless function $\Phi$ with $l$ arguments. Specifically, if we can construct a dimensionless quantity $\Pi$ using the quantity of interest $a$ and $l$ independent dimensionless quantities $\Pi_1, \dots, \Pi_l$ from the governing parameters, then we can write:

    \begin{equation}
        \Pi = \Phi \left( \Pi_1, \dots, \Pi_l \right)
    \end{equation}
\end{nonameteo*}

The dimensionless construction that we performed earlier is the one proposed by Barenblatt\cite{barenblatt_2003}. However, it is important to emphasize that this is just one of several possible construction choices, and the function $\Phi$ in Buckingham's $\Pi$-theorem will depend on this choice. A final but useful observation is that the $\Pi$-theorem is intuitively obvious. This is because physical laws cannot depend on the choice of units and, therefore, it should be possible to express them using relationships between quantities that do not depend on this arbitrary choice (dimensionless quantities). As a historical note, this fact was recognized long before Buckingham's \(\Pi\)-theorem was formally formulated and proved. Pioneers such as Galileo, Newton, Fourier, and Maxwell used concepts of dimensional analysis in their work\cite{korner,barenblatt_2003}.

A simple application of this dimensional analysis framework is the (Fanning) friction factor \( f \) defined in \eqref{eq:friction_def}. Notice that \( f \) can be expressed as a function of the following parameters:
$f = F \left(\rho, D, \mu, \bar{U} \right)$, and, since $f$ is dimensionless, by applying Buckingham's \(\Pi\)-theorem, we can construct the Reynolds number \(\Pi_1 = Re = \rho \bar{U} D / \mu\) and show that the friction factor is purely a function of \( Re \), that is, $ f = \Phi \left( Re \right)$.
As mentioned in section \ref{sec:intro}, Nikuradse\cite{Nikuradse} conducted experiments to measure the friction factor in sand-roughened pipes using well-defined size sand grains \( r \). In this case, we must account for the newly introduced parameter and incorporate it into another dimensionless quantity, \(\Pi_2 = r / D\). This modifies the friction equation to

\begin{equation}\label{eq:rough_pipes}
    f = \Phi \left(Re, r/D \right).
\end{equation}

One of the most useful applications of the $\Pi$-theorem is experimental design. Several experiments aim to recreate real-world behavior in smaller and more controlled environments. This is achieved by properly rescaling our governing parameters in such a way that the dimensionless quantities $\Pi, \Pi_1, \dots, \Pi_l$ remain unchanged. However, an experiment only serves a purpose if we can make predictions about the phenomena at hand based on measurements made in this "prototype" scenario, so that we must know exactly which rescaling transformations achieve such goal. Such a group of transformations is what we call \textbf{Buckingham's Similarity Group} and, by following Barenblatt's dimensionless construction, it can be written as

\begin{align} \label{eq:buckingham_group_barenblatt}
    & a_1^* = A_1 a_1; \ \ \ \ a_2^* = A_2 a_2; \ \ \ \ \dots \ \ \ \ ; a_m^* = A_m a_m; \nonumber \\ & \nonumber \\
    &b_1^* = A_1^{-\alpha_1^{(1)}} \dots A_m^{-\alpha_m^{(1)}} b_1; \ \ \ \ \dots \ \ \ \ b_l^* = A_1^{-\alpha_1^{(l)}} \dots A_m^{-\alpha_m^{(l)}} b_l; \\ & \nonumber \\
    & a^* = A_1^{-\alpha_1} \dots A_m^{-\alpha_m} a, \nonumber
\end{align}
where $A_1, \dots, A_m$ are arbitrary positive constants. In the context of smooth pipe flow, dimensional analysis suggests that, for example, we could double the flow rate \(\bar{U}\) while halving the diameter of the pipe. This would ensure that both the friction factor \(f\) and the Reynolds number \(Re\) remain constant, thereby guaranteeing that the flow behavior remains similar. For rough pipes, the roughness must be rescaled proportionally to the diameter according to \eqref{eq:rough_pipes}. It is important to note that the scaling laws \(f \sim \frac{1}{Re}\) for laminar flows and \(f \sim \frac{1}{Re^{1/4}}\) for turbulent flows are not derived from Buckingham's Theorem.

\section{Similarity Beyond Buckingham} \label{sec:beyond_buckingham}

While the techniques presented in the previous section are well known to the scientific community, particularly to those specializing in fluid dynamics, we chose to include a brief introduction as some of the intricacies of similarity groups and experimental design are often overlooked. What may not be as widely recognized is the possibility of taking a further step in the study of dimensionless similarity. One of the key advantages of Buckingham's \(\Pi\)-theorem is its ability to reduce dimensionality, transforming a function \(F\) with \(m+l\) arguments into a function \(\Phi\) with \(l\) arguments.  This further step proposes an even greater reduction in dimensionality.

Some phenomena exhibit what we call complete or incomplete similarity, which will be formally defined later in this section. This new type of similarity induces another set of parameter transformations, which can further enhance our experimental design. A famous example of incomplete similarity was presented by Goldenfeld\cite{Goldenfeld2006} in the context of Nikuradse's rough pipes. As explained in section \ref{sec:intro}, he provided empirical evidence showing that equation \eqref{eq:rough_pipes} could be further simplified to:

\begin{equation} \label{eq:roughness_goldenfeld_similarity}
    f= Re^{-1/4} \Phi^{(1)} \left( Re^{3/4} \left( r / D \right) \right).
\end{equation}

To formally define both complete and incomplete similarity, we return to Buckingham's \(\Pi\)-theorem and express the dimensionless equation it provides:

\begin{equation}
    \Pi = \Phi \left( \Pi_1, \dots, \Pi_l \right).
\end{equation}
Following Barenblatt\cite{barenblatt_2003}, suppose now there exists a non-zero limit of the function $\Phi$ when the parameters $\Pi_{n+1}, \dots, \Pi_l$ all go to zero or infinity while the other similarity parameters $\Pi_1, \dots, \Pi_n$ remain constant.  If this convergence is sufficiently fast, then for small or large values of \(\Pi_{n+1}, \dots, \Pi_l\), the function \(\Phi\) can be approximated by a function of fewer arguments:
\begin{equation} \label{eq:complete_similarity_definition_introduction}
	\Pi = \Phi^{(0)} \left( \Pi_1, \dots, \Pi_n \right).
\end{equation}

In such cases, we say that the phenomenon exhibits \textbf{complete similarity} or \textbf{similarity of the first kind}.

\begin{example*}[The log law]
	Recall that if we are working with Newtonian flows in wall units, Buckingham's $\Pi$-Theorem allows us to write
	\begin{equation}
		u^+ = \Phi\left( y^+, Re_\tau \right),
	\end{equation}
	where $u^+:=U/u_\tau$ is the well-known dimensionless velocity, $Re_{\tau}:=\rho u_{\tau}D/\mu$ is the friction Reynolds number, and $y^+:=y/\delta_\nu:=y\rho u_\tau/\mu$ is the distance from the wall in wall units\cite{White,Pope_2000}. One of the cornerstones of turbulence theory is the existence of a logarithmic layer, where the function \(\Phi\) becomes independent of the friction Reynolds number \(Re_\tau\). This leads to the well-known logarithmic law of the wall:
\begin{equation}
u^+ = \Phi^{(0)}(y^+) = \frac{1}{\kappa} \log y^+ + B,
\end{equation}
where \(B\) is an integration constant and \(\kappa\) is the von K\'arm\'an constant. It is important to emphasize that the location of the logarithmic layer and the validity of the logarithmic law remain active areas of research\cite{KLEWICKI_FIFE_WEI_2009, Bailey_Vallikivi_Hultmark_Smits_2014, George2007}.
\end{example*}

If complete similarity is identified in a physical system, it significantly simplifies the modeling of specific cases where the dimensionless parameters are either very small or very large. Although several examples of this type of similarity exist in the literature, it is far from being the most common scenario.

Typically, when the dimensionless governing parameters \(\Pi_{n+1}, \dots, \Pi_l\) approach zero or infinity, the function \(\Phi \left( \Pi_1, \dots, \Pi_n, \Pi_{n+1}, \dots, \Pi_l \right)\) does not necessarily tend to a finite, non-zero limit, such as in \eqref{eq:complete_similarity_definition_introduction}. Instead, there exists a broader class of phenomena than those that exhibit complete similarity. For this wider class, the function \(\Phi\) demonstrates generalized homogeneity in its dimensionless arguments at extreme values of \(\Pi_{n+1}, \dots, \Pi_l\):

\begin{align} \label{eq:incomplete_similarity_definition_introduction}
	\Pi = \Pi_{n+1}^{-\xi_{n+1}} \cdots \Pi_l^{-\xi_l} \Phi^{(1)} \left( \Pi_1', \dots, \Pi_n' \right),
\end{align}
where
\begin{align} \label{eq:auxiliary_incomplete_similarity_definition}
    \Pi_1' &= \Pi_1 \cdot \Pi_{n+1}^{\xi_{n+1}^{(1)}} \cdots \Pi_l^{\xi_l^{(1)}}, \\ & \vdots \nonumber \\ \Pi_n' &= \Pi_n \cdot \Pi_{n+1}^{\xi_{n+1}^{(n)}} \cdots \Pi_l^{\xi_l^{(n)}}, \nonumber
\end{align}
and  $\xi_i$ and $\xi_i^{(j)}$ are constants. If that is the case, we say that the phenomenon has the property of \textbf{incomplete similarity} or \textbf{similarity of the second kind}. It should be noted that the generalized homogeneity of the function $F$ in Buckingham's $\Pi$-theorem follows from the general physical covariance principle, and the exponents are obtained by simple rules of dimensional analysis. In contrast, the generalized homogeneity related to incomplete similarity is a special property of the problem under consideration. Therefore, exponents $\xi_i$ and $\xi_i^{(j)}$ cannot be obtained by using only dimensional analysis.

However, if one manages to find an incomplete similarity, the property of generalized homogeneity induces another group of transformations that retain the properties of the phenomena. We call it the \textbf{Renormalization Group}, and the process of its construction is similar to the Buckingham similarity group. Its deduction can be found in the appendix, but if one follows Barenblatt's dimensionless construction, it can be written as:
\begin{align} \label{renormalization_group_barenblatt}
	& a_1^* = a_1; \ \ \ \ a_2^* = a_2; \ \ \ \ \dots \ \ \ \ ; a_m^* = a_m; \\ & \nonumber \\
	& b_1^* = B_{n+1}^{-\xi_{n+1}^{(1)}} \dots B_l^{-\xi_l^{(1)}} b_1; \ \ \ \ \dots \ \ \ \ b_n^* = B_{n+1}^{-\xi_{n+1}^{(n)}} \dots B_l^{-\xi_l^{(n)}} b_n; \nonumber \\ & \nonumber \\ & b_{n+1}^* = B_{n+1} b_{n+1}; \ \ \ \ b_{n+2}^* = B_{n+2} b_{n+2}; \ \ \ \ \dots \ \ \ \ ; b_l^* = B_l b_l; \nonumber \\ & \nonumber \\
	& a^* = B_{n+1}^{-\xi_{n+1}} \dots B_{l}^{-\xi_l} a, \nonumber
\end{align}
where $B_{n+1}, \dots, B_l$ are arbitrary positive constants. Whether a phenomenon possesses the property of complete or incomplete similarity depends not only on the problem itself but upon our own choice of what we call its governing parameters. This fact is not trivial, and thus a simple yet elucidating example will be provided in the next section.

\section{Example: Laminar Newtonian Flows in Bulk and Pressure Coordinates} \label{sec:laminar_newtonian}

In channel flows, by fixing the half-lenght of the channel  $\delta$, the fluid density \( \rho \), the fluid viscosity \( \mu \), and the distance from the wall \( y \), it is not possible to independently prescribe both the pressure gradient \(- dP/dz\) and the flow rate \(\bar{U}\). The friction factor \( f \) is what connects \(- dP/dz\) to \(\bar{U}\). Given this relationship, in any dimensional analysis of Newtonian flow, we must choose whether to use \(\bar{U}\) or \(- dP/dz\) as our final governing parameter, as either one, combined with the parameters mentioned above, completely determines the flow \( U \). The selection of governing parameters is sometimes referred to as choosing the dimensional coordinates of a system. In our context, these are termed \textbf{bulk velocity coordinates} and \textbf{pressure drop coordinates}. To be more precise, $\left(y, \rho, \mu, \delta, \bar{U}\right)$ are the bulk velocity coordinates and $\left(y, \rho, \mu, \delta, - dP/dz\right)$ are the pressure drop coordinates.

It should come as no surprise that the dimensionless quantities we can construct will also greatly depend on our choice of governing parameters. When working in bulk velocity coordinates, we can write, with the aid of Buckingham's \(\Pi\)-theorem:
\begin{equation}
    \frac{U}{\bar{U}} = \Phi \left( \frac{y}{\delta}, Re \right).
\end{equation}

 Because we are dealing with laminar flows, the dimensionless function $\Phi$ can be explicitly derived as:
\begin{equation}
    \frac{U}{\bar{U}} = 3 \left( \frac{y}{\delta} \right) - \frac{3}{2} \left( \frac{y}{\delta} \right)^2 = \Phi^{(0)}\left( \frac{y}{\delta} \right).
\end{equation}
The equation above implies, in particular, that the laminar Newtonian flows have the property of complete similarity when expressed in bulk velocity coordinates, i.e. the bulk normalized velocity depends only on one of the dimensionless parameters, $y/\delta$, and there is no explicit dependence on $Re$ for laminar flows (which we can think of as a Newtonian flow for small enough $Re$).
On the other hand, when working in pressure drop coordinates, another set of dimensionless parameters emerge:
\begin{equation}
    u^+ = \Psi \left( y^+, Re_\tau \right),
\end{equation}
Although there are well-known implicit definitions, we prefer to explicitly describe how to construct the dimensionless quantities \( u^+ \), \( y^+ \), and \( Re_\tau \) directly from the governing parameters for the sake of clarity:
\begin{align}
    u^+ & = \frac{U \rho^{1/2}}{\left(- dP / dz \right)^{1/2} \delta^{1/2}}; \\ & \nonumber \\ y^+ &= \frac{y \rho^{1/2} \delta^{1/2} \left( - dP / dz \right)^{1/2}}{\mu}; \nonumber \\ & \nonumber \\ Re_\tau &= \frac{\left( - dP / dz \right)^{1/2} \rho^{1/2} \delta^{3/2}}{\mu}. \nonumber
\end{align}
When writing $\Psi$ explicitly in its most usual form, it reads:
\begin{equation}
    u^+ = \Psi \left(y^+, Re_\tau \right) = y^+ - \frac{1}{2 Re_\tau} \left( y^+ \right)^2.
\end{equation}
However, we can make use of some algebraic manipulation in order to see that this is indeed a case of incomplete similarity in pressure drop coordinates, as:
\begin{equation} \label{eq:laminar_newtonean_incomplete}
    u^+ = Re_\tau \left( \frac{y^+}{Re_\tau} - \frac{1}{2} \left( \frac{y^+}{Re_\tau} \right)^2 \right) = Re_\tau \Psi^{(1)} \left( \frac{y^+}{Re_\tau} \right).
\end{equation}
The renormalization group induced by such similarity can be written as:

\begin{align}
\label{eq:laminar_renormalization_group}
    & \mu^* = \mu, \ \ \rho^* = \rho, \ \ \delta^* = \delta, \ \ y^* = y, \\ & \nonumber \\  & - \frac{dP}{dz}^* = B_1 \left( - \frac{dP}{dz} \right), \ \ \ \ U^* = B_1 U. \nonumber 
\end{align}
It is important to emphasize that the incomplete similarity identified above is not purely derived from dimensional analysis, as it required us to know the explicit formula for the function \(\Psi\). We remark that, despite being commonly used to study turbulent flows near the wall, the dimensionless quantities $u^+$, $y^+$ and $Re_\tau$ can be used to analyze any kind of Newtonian flow, and their usage here unveils an important incomplete similarity relation valid in the entire domain of the flow.

Because it might appear to be a generalization of Buckingham's \(\Pi\)-theorem, one could be tempted to think that discovering incomplete similarity exponents is a trivial application of linear algebra. However, on closer inspection, these exponents are not constrained by any dimensional laws. They are generalized homogeneity exponents for a function that is already dimensionless, meaning that they can be arbitrary real numbers. This makes their discovery even more challenging and dependent on the specific properties of the phenomena under study.

Additionally, the laminar Newtonian example demonstrates that the presence of similarity depends on both the chosen model and the governing parameters of the phenomenon. Although this may seem like a straightforward statement, it is not always easily internalized, and we encourage the reader to reflect on its implications.

\section{Generalized Similarity Groups} \label{sec:generalized_similarity_groups}

The reader may have noticed that the dimensionless quantity $y^+$ introduced in the last section does not fit the classical construction proposed by Barenblatt as in equation \eqref{eq:barenblatt_construction}, as we make use of the two dimensionally dependent parameters $y$ and $-dP/dz$ instead of just one. If we had access to data in dimensional form, this would not be a problem at all, because we could just change the dimensionless construction to fit the classical theory.

However, data are often available already in dimensionless form, and most of the dimensionless quantities provided already have good intuitive reasons to be defined as such. This can be even more problematic when we realize that both Buckingham's Similarity Group and the Renormalization Group presented in equations \eqref{eq:buckingham_group_barenblatt} and \eqref{renormalization_group_barenblatt} respectively, are highly dependent on our choice of dimensionless construction. Thus, it is important to generalize the dimensionless construction and derive the similarity groups associated with it. We present this new construction in this section, choosing to call it the Multiple Dimensionally Dependent Parameters construction (MDDP for short).

All the mathematical proofs regarding MDDP and its associated similarity groups are available in the appendix, as we choose to present in this section only the main definitions and results. To the best of the authors' knowledge, this is the first work to present such definitive generalization in a mathematically complete framework.

MDDP is obtained by simply generalizing the construction equation \eqref{eq:barenblatt_construction} in the following way:
\begin{align}
    \Pi_j &= b_1^{\beta_1^{(j)}} \dots b_l^{\beta_l^{(j)}} a_1^{\alpha_1^{(j)}} \dots a_m^{\alpha_m^{(j)}}, \ \ \ \ \ \  j=1, \dots, l \\ & \nonumber \\ \Pi &= a^\beta b_1^{\beta_1} \dots b_l^{\beta_l} a_1^{\alpha_1} \dots a_m^{\alpha_m}, \nonumber
\end{align}
where we choose the exponents above in order to make the $\Pi_j$'s and $\Pi$ dimensionless. It is possible to show that for every choice of exponents vector $\beta^{(j)} \in \R^l$ there is an unique choice of exponents vector $\alpha^{(j)} \in \R^m$ which makes $\Pi_j$ dimensionless and, moreover, we also ask the vectors $\left\{ \beta^{(j)}\right\}_{j=1,\dots,l}$ to be linearly independent, as this implies the dimensionless quantities $\Pi_j$ 
to be independent by exponentiation and multiplication, i.e. we are not able to construct a $\Pi_j$ through other dimensionless quantities in the following way:
\begin{equation}
    \Pi_j = \Pi_1^{\gamma_1} \dots \Pi_{j-1}^{\gamma_{j-1}} \Pi_{j+1}^{\gamma_{j+1}} \dots \Pi_l^{\gamma_l}.
\end{equation}
In the MDDP construction, we do not limit ourselves to just one dimensionally dependent governing parameter $b_j$ when constructing each $\Pi_j$, as done by the Barenblatt and illustrated in equation \eqref{eq:barenblatt_construction}. While it is always possible to make a construction in Barenblatt's sense, we must be able to calculate similarity and renormalization groups if we are presented with more general dimensionless quantities.

Next, we present two results, with their corresponding proofs provided in Appendix \ref{sec:mddp}. The first result is related to the construction of Buckingham's Similarity Group in this new setting. What we are looking for is a set of exponents $\delta_1, \dots, \delta_m, \delta_j^{(i)}$ with $1\leq i \leq l$ and $1\leq j \leq m$ such that the following transformation does not change the value of the dimensionless quantities $\Pi, \Pi_1, \dots, \Pi_l$:
\begin{align}
    a_1^* &= A_1 a_1; \ \ \ \ a_2^* = A_2 a_2; \ \ \dots \ \ ; a_m^* = A_m a_m; \\ & \nonumber \\   b_1^* &= A_1^{\delta_1^{(1)}} \dots A_m^{\delta_m^{(1)}} b_1; \ \ \ \ \ \  \dots \ \ \ \ \ b_l^* = A_1^{\delta_1^{(l)}} \dots A_m^{\delta_m^{(l)}} b_l; \nonumber \\ & \nonumber \\
    a^* &= A_1^{\delta_1} \dots A_m^{\delta_m} a, \nonumber
\end{align}
for arbitrary positive constants $A_1, \dots, A_m$. The result is as follows:

\begin{nonameteo*}[\textbf{Claim I}]
    The exponents $\delta_j^{(i)}$ are found by solving the $m$ following linear systems:
    \begin{equation} \label{eq:buckinghams_similarity_teo_main_text}
            \begin{bmatrix}
        \beta_1^{(1)} & \cdots & \beta_l^{(1)} \\
        \vdots & \ddots & \vdots \\
        \beta_1^{(l)} & \cdots & \beta_l^{(l)}
    \end{bmatrix}
    \begin{bmatrix}
        \delta_j^{(1)} \\
        \vdots \\
        \delta_j^{(l)}
    \end{bmatrix} =
    \begin{bmatrix}
        -\alpha_j^{(1)} \\
        \vdots \\
        -\alpha_j^{(l)}
    \end{bmatrix}, \ \ \ \  \ j=1, \dots, m.
    \end{equation}
    Furthermore, once the values of the $\delta_j^{(i)}$ are found, each $\delta_j$ can be computed by:
    \begin{equation}
        \delta_j = - \left( \frac{\alpha_j + \sum_{i=1}^l \delta_j^{(i)} \beta_i}{\beta} \right), \ \ \ \ \  j=1, \dots, m.
    \end{equation}
\end{nonameteo*}
After establishing this first result, the natural question is whether a similar theorem can be deduced for the renormalization group if our phenomena happens to have the property of incomplete similarity. Assuming we indeed find such similarity as in equations \eqref{eq:incomplete_similarity_definition_introduction} and \eqref{eq:auxiliary_incomplete_similarity_definition}, we now look for a set of exponents $\mu_{n+1}, \dots, \mu_l, \mu_j^{(i)}$ with $1\leq i \leq n$ and $n+1\leq j \leq l$ such that the following transformation does not change the values of the renormalized dimensionless quantities $\Pi_1', \dots, \Pi_n'$ and $\Pi':=\Pi \cdot \Pi_{n+1}^{\xi_{n+1}} \cdots \Pi_l^{\xi_l}$:
\begin{align}
    a_1^* &=  a_1; \ \ \ \ a_2^* = a_2; \ \ \dots \ \  a_m^* = a_m; \\ & \nonumber \\   b_1^* &= B_{n+1}^{\mu_{n+1}^{(1)}} \dots B_l^{\mu_l^{(1)}} b_1; \ \ \ \ \ \  \dots \ \ \ \ \ b_n^* = B_1^{\mu_{n+1}^{(n)}} \dots B_l^{\mu_l^{(n)}} b_n; \nonumber \\ & \nonumber \\ b_{n+1}^* &= B_{n+1} b_{n+1}; \ \ \ \ b_{n+2}^* = B_{n+2} b_{n+2}; \ \ \dots \ \  b_l^* = B_l b_l; \nonumber \\ & \nonumber \\   
    a^* &= B_{n+1}^{\mu_{n+1}} \dots B_l^{\mu_l} a, \nonumber
\end{align}
for arbitrary positive constants $B_{n+1}, \dots B_l$. The result we derived is also based on the solution of several linear systems, and it reads:
\begin{nonameteo*}[\textbf{Claim II}]
    The exponents $\mu_j^{(i)}$ are found by solving the $l - n$ following linear systems:
    \begin{align} \label{eq:renormalization_group_teo_main_text}
    &\begin{bmatrix}
    \text{---} & \beta_{\left\{ 1, \dots n \right\}}^{(1)} + \sum_{k=n+1}^l \xi_k^{(1)} \beta_{\left\{1, \dots, n \right\}}^{(k)} & \text{---} \\
    \vdots & \vdots & \vdots \\
    \text{---} & \beta_{\left\{ 1, \dots n \right\}}^{(n)} + \sum_{k=n+1}^l \xi_k^{(n)} \beta_{\left\{1, \dots, n \right\}}^{(k)} & \text{---}
    \end{bmatrix}
    \begin{bmatrix}
    \vrule \\ \ \\
    \mu_{j}^{\left( 1, \dots, n \right)} \\ \ \\
    \vrule
    \end{bmatrix} = \\  & = - 
    \begin{bmatrix}
    \beta_{j}^{(1)} + \sum_{k=n+1}^l \xi_k^{(1)} \beta_{j}^{(k)} \\ 
    \vdots \\
    \beta_{j}^{(n)} + \sum_{k=n+1}^l \xi_k^{(n)} \beta_{j}^{(k)}
    \end{bmatrix}, \ \ \ \ j=n+1,\dots, l. \nonumber
\end{align}
Where we introduced some new notation in the equation above in the form of  $\beta_{\left\{ 1, \dots, n\right\}}^{(k)} := \left( \beta_1^{(k)}, \dots, \beta_n^{(k)} \right)$ and $\mu_j^{\left( 1, \dots, n \right)} := \left( \mu_j^{(1)}, \dots, \mu_j^{(n)} \right)$. Furthermore, once the values of the $\mu_j^{(i)}$ are found, each $\mu_j$ can be computed by:
\begin{align}
    \mu_j = & - \frac{1}{\beta}\left\langle \beta_{\left\{ 1, \dots, n\right\}} \ , \ \mu_j^{\left( 1, \dots\ n \right)} \right\rangle  \\  & -  \frac{1}{\beta}\left\langle \sum_{k=n+1}^l \xi_k \beta_{\left\{ 1, \dots, n \right\}}^{(k)} \ , \ \mu_j^{(1, \dots, n)} \right\rangle \nonumber \\ & - \frac{1}{\beta} \left(\beta_j + \sum_{k=n+1}^l \xi_k \beta_j^{(k)} \right), \nonumber
\end{align}
for $j=n+1, \dots, l$, where the notation $\langle \cdot, \cdot \rangle$ above denotes the usual scalar inner product.
\end{nonameteo*}

For example, the scaling relation for rough pipes in \eqref{eq:goldenfeld_renormalization_group} can be explicitly written as a renormalization group:
\begin{align}
        & \mu^* = \mu, \ \ \rho^* = \rho, \ \ D^* = D, \\ & \nonumber \\  & \bar{U}^* = B_1 \bar{U}, \ \ \ r^*= B_1^{-3/4} r, \ \ \ f^* = B_1^{-1/4} f. \nonumber 
\end{align}
Such group can indeed be found by remembering that $(\Pi, \Pi_1, \Pi_2) = (f, r/D, Re)$, and by applying the renormalization group theorem to the incomplete similarity relation
\begin{equation}
    \Pi = \Pi_2^{-1/4} \Phi^{(1)} \left( \Pi_1 \Pi_2^{3/4} \right).
\end{equation}
Another example concerns the renormalization relation for laminar flows in friction coordinates obtained in \eqref{eq:laminar_renormalization_group}, which can also be derived with \((\Pi, \Pi_1, \Pi_2) = (u^+, y^+, Re_\tau)\). By applying the renormalization group theorem to the incomplete similarity relation found in this context, which can be written as
\begin{equation}
    \Pi = \Pi_2 \Psi^{(1)} \left( \Pi_1 \Pi_2^{-1} \right).
\end{equation}

\section{Learning similarity from data}\label{sec:Barenet}

We have seen examples of similarities and their associated similarity groups in the context of fluid flows. Understanding the intricacies of these problems was crucial for these discoveries. In the following sections, we propose a different approach, aimed at answering a simple question: ``If we have experimental data for any given phenomenon, already collected in non-dimensional form, can we construct a machine learning framework to discover incomplete similarity?"

This question leads to several further inquiries. The first is: "Why should we assume that experimental data will already be collected in non-dimensional form?". There are two main reasons for this assumption. First, most experimentalists are well aware of Buckingham's $\Pi$-Theorem and the power of dimensionless representation, so they typically provide experimental data in this form. Second, even if we encounter data collected in dimensional form, there is excellent work by Bakarji \emph{et al.} \cite{bakarji2022}, presenting several numerical methods to transform dimensional data into dimensionless parameters using the $\Pi$-Theorem framework. In particular, they introduce the \textbf{Buckinet}, a neural network with an architecture that greatly inspired the creation of our neural network that discovers incomplete similarities from dimensionless data. 

The second question is ``Why focus specifically on incomplete similarity and not both complete and incomplete similarity?" The answer is straightforward: complete similarity is a particular instance of the incomplete case with some exponents equal to zero. Thus, if we identify the appropriate zero exponents when examining the exponents found by any numerical method for incomplete similarity, we can conclude that we are studying phenomena within the realm of complete similarity.

Having addressed some of the possible questions, we now present the main result of this section: the formulation of a neural network (Barenet) that checks if the phenomena represented by the provided dimensionless data exhibit incomplete similarity and, if so, determines the suitable exponents of this similarity (see Equation \eqref{eq:incomplete_similarity_definition_introduction}). The network is built inside a Python package \cite{python3_manual}, which
also uses the theorems in Section \ref{sec:generalized_similarity_groups} to explicitly calculate both Buckingham's similarity group (derived from the dimensionless construction provided) and the renormalization group (derived from the incomplete similarity exponents if they are found). 
Our package is available at our Github repository\cite{Sanfins2024}. Despite being under development, it includes a series of tutorial notebooks that guide users through generating the examples presented later in this manuscript.

\begin{figure*}%[!h]
	\includegraphics[scale=0.42, angle=270]{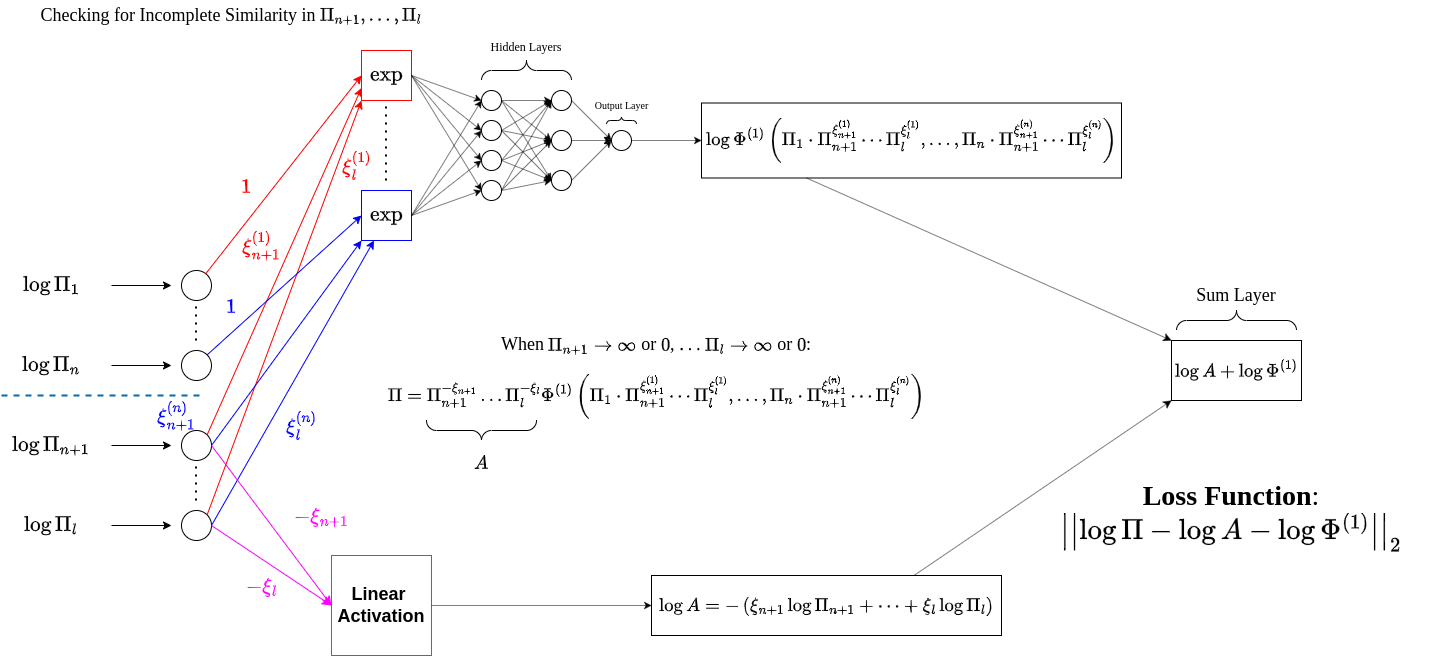}
	\centering
	\caption{Schematics of the Barenet's architecture.}
	\label{fig:Barenet_main_text}
\end{figure*}

Recall that a phenomenon is said to exhibit incomplete similarity if it satisfies the equation \eqref{eq:incomplete_similarity_definition_introduction}. The exponents $\xi_i^{(j)}$ and $\xi_i$, as well as the function $\Phi^{(1)}$, are unknown. Thus, if our neural network can identify a set of exponents and determine a function $\Phi^{(1)}$ (discovered by the neural network itself) that satisfies the equation, it can be concluded that the phenomenon exhibits incomplete similarity.

The neural network features a straightforward architecture, depicted in Fig. \ref{fig:Barenet_main_text}. It processes data in the form of natural logarithms of the dimensionless quantities provided by the MDDP construction. The incomplete similarity exponents are embedded as part of the weights in the first layer. An exponential activation layer calculates the arguments of the function $\Phi^{(1)}$, while a parallel linear activation layer computes the natural logarithm of the quantity $\Pi_{n+1}^{-\xi_{n+1}} \cdots \Pi_l^{-\xi_l}$, which multiplies the function. This quantity is denoted as $A$.

\begin{align}
	\log A & = \log \left( \Pi_{n+1}^{-\xi_{n+1}} \cdots \Pi_l^{-\xi_l} \right) = \\ & \nonumber \\ & = - \left( \xi_{n+1} \log \Pi_{n+1} + \dots + \xi_l \log \Pi_l \right). \nonumber
\end{align}

After that, the arguments of $\Phi^{(1)}$ go trough a dense neural network in order to properly approximate $\log \Phi^{(1)}$. The network's architecture proceeds by passing the two parallel quantities into a sum layer, so that the final result will aim to recreate the natural logarithm of the right-hand side of \eqref{eq:incomplete_similarity_definition_introduction}. Furthermore, since our goal is to approximate a function, we can employ the Euclidean loss function  $\left| \left| \log \Pi - \log A -\log \Phi^{(1)} \right| \right|_2$, where the choice of comparing the logarithms was made purely for reasons of stability in the optimization methods.

We will now present some examples in which our network found incomplete similarity from data. Training for all the following examples was done using the TensorFlow \cite{tensorflow2015-whitepaper} framework with the standard Adam optimizer.

The first example are laminar Newtonian flows in pressure drop coordinates. We can refer to equation \eqref{eq:laminar_newtonean_incomplete} to see that the mean velocity profile (MVP for short) in pressure drop coordinates can be written in the form:
\begin{equation}
	u^+ = y^+ - \frac{1}{2 Re_\tau} \left( y^+ \right)^2.
\end{equation}
We proceeded to generate samples of this velocity profile for $100$ different values of $Re_\tau$ in the interval $[10, 100]$. Moreover, it was already discussed in Section \ref{sec:laminar_newtonian} that we can algebraically manipulate such an equation to fit our definition of incomplete similarity. Indeed:
\begin{equation}
	u^+ = Re_\tau \left( \frac{y^+}{Re_\tau} - \frac{1}{2} \left( \frac{y^+}{Re_\tau} \right)^2 \right) = Re_\tau \Phi^{(1)} \left( \frac{y^+}{Re_\tau} \right),
\end{equation}
where $\Phi^{(1)} \left( \omega \right) = \omega - \omega^2/2$. Let us now rewrite what we mean by incomplete similarity just so that the reader can compare both cases:
\begin{equation}\label{eq:incomplete_similairty_general_Barenet}
	u^+ = Re_\tau^{-\xi_2} \Phi^{(1)} \left( y^+ Re_\tau^{\xi_2^{(1)}} \right).
\end{equation}
We can see that the laminar Newtonian MVP has such a similarity with $\xi_2 = \xi_2^{(1)} = -1$. Using our package, we can try to find these exponents using our network and, after training for $1000$ epochs, it estimated such values as $\hat{\xi}_2 = -0.97$ and $\hat{\xi}_2^{(1)} = -0.95$. Plots of both the raw data and the renormalized collapsed data i.e. a plot in the coordinates $y^+ Re_\tau^{\hat{\xi}_2^{(1)}} \times u^+ Re_\tau^{\hat{\xi}_2}$, can be found in figures \ref{fig:laminar_mvp_wall_coordinates} and \ref{fig:laminar_mvp_wall_coordinates_renormalized}. Moreover, once an incomplete similarity is found, we calculate and store the renormalization group using its associated theorem.

The second example revisits Nikuradse's roughness data to confirm Goldenfeld's exponents, as mentioned in the introduction and depicted in Figures \ref{fig:nikuradse_data} and \ref{fig:nikuradse_data_goldenfeld_exponents}. Referring to equation \eqref{eq:roughness_goldenfeld_similarity}, incomplete similarity should be identified as:
\begin{equation}
    f= Re^{-\xi_2} \Phi^{(1)} \left( \left( r / D \right) Re^{\xi_2^{(1)}} \right),
\end{equation}
with $\xi_2 = 1/4$ and $\xi_2^{(1)} = 3/4$. After training for $10000$ epochs due to the few available data points, the network estimated $\hat{\xi}_2 = 0.22$ and $\hat{\xi}_2^{(1)} = 0.78$. A plot of the renormalized data can be found in Figure \ref{fig:nikuradse_renormalized} for comparison with the exponents proposed by Goldenfeld.

The methods presented in this work open numerous avenues for further exploration, as scaling and similarity relations are fundamental in fluid mechanics. These range from basic dimensional analysis \cite{Taylor1,Taylor2} to advanced multiscale frameworks \cite{Wei2005, KLEWICKI_FIFE_WEI_2009, klewicki2013}, as well as Lie symmetry transformations and invariant solutions \cite{Hoyas_Oberlack_PRL, Hoyas_Oberlack_PRF}. We intend to establish connections between the current work and these approaches in future studies.

It is also worth noting that the methods discussed here can be applied to various physical systems, such as critical phenomena and phase transitions in complex systems such as ferromagnetism\cite{Lectures_Goldenfeld}, or even in purely mathematical contexts, such as fractal geometry \cite{barenblatt_2003}. Our focus on fluid dynamics stems from the critical role that similarity and scaling play in this field.

In the following section, we discuss incomplete similarity relations for laminar Herschel-Bulkley fluid flows. Although this example is significant on its own for its application in the science and industry of viscoplastic flows, its inclusion in this work is crucial for exploring the limitations and possibilities of using our machine learning framework. Specifically, we demonstrate its full power by performing a non-trivial dimensionality reduction of the dimensionless parameters for the velocity profile from three dimensions to two. This contrasts with the previous examples, where the reduction in dimensionality was from two dimensions to one. 

\begin{figure*}[htp]
        \centering
	\includegraphics[width=17cm]{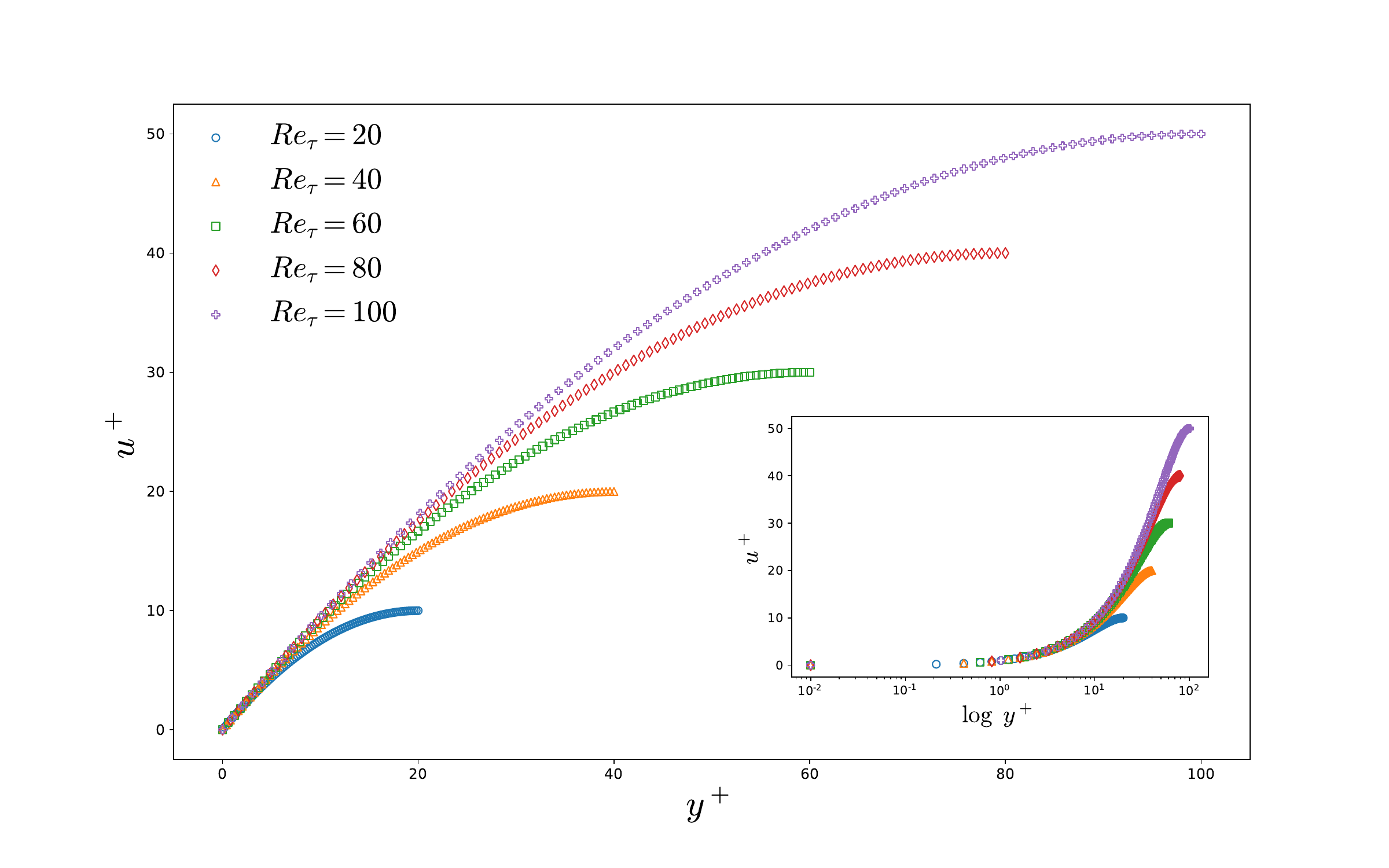}
	\caption{Laminar MVP in pressure drop coordinates.}
	\label{fig:laminar_mvp_wall_coordinates}
\end{figure*}

\begin{figure*}[htp]
	\includegraphics[width=17cm]{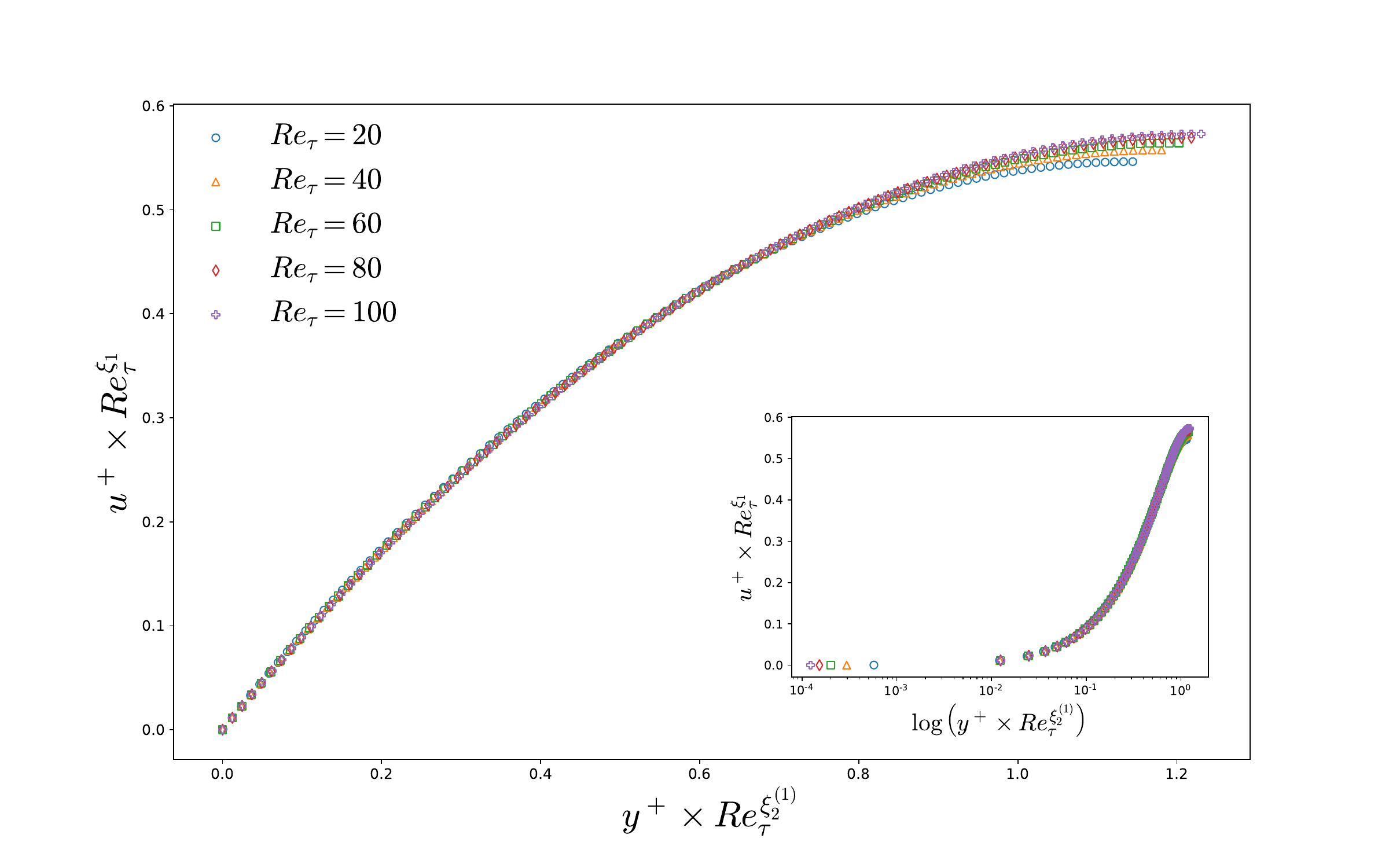}
	\centering
	\caption{Laminar MVP after renormalization with exponents found by our neural network.}
	\label{fig:laminar_mvp_wall_coordinates_renormalized}
\end{figure*}

\begin{figure*}[htp]
	\includegraphics[width=17cm]{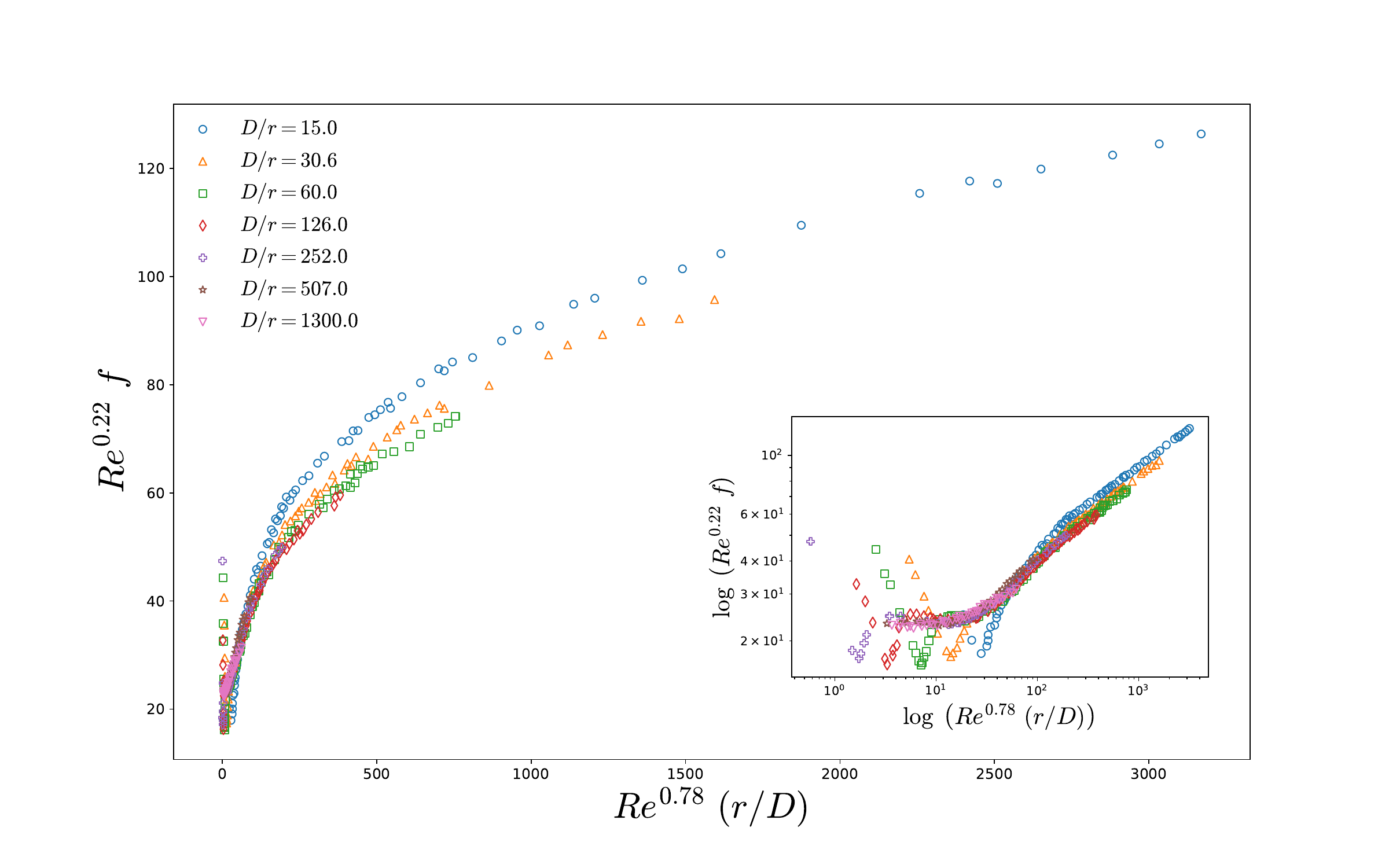}
	\centering
	\caption{Nikuradse's roughness data after renormalization with exponents found by our neural network.}
	\label{fig:nikuradse_renormalized}
\end{figure*}

\section{Laminar Herschel-Bulkley Flows: A non-trivial example of incomplete similarity.} \label{sec:HB_flows}

 This section provides an important example within the context of similarity laws and renormalization groups, with a focus on fluid dynamics. Specifically, we analyze the Herschel-Bulkley model, where flow dynamics is governed by a yield stress parameter $\tau_0$ and a shear rate-dependent viscosity expressed as $\tau_{rz}=\tau_0+K\left(-\frac{dU}{dr}\right)^n$ for $\tau_{rz}>\tau_0$, where $r$ denotes the radial direction and $z$ denotes the axial direction, with $K$ representing the fluid consistency coefficient and $n$ the flow behavior index. Through this example, we illustrate how these fluids adhere to established similarity laws for friction factors and velocity profiles, alongside their associated Buckingham similarity groups. Additionally, we demonstrate the existence of an incomplete similarity relation and the associated renormalization group for both quantities. Importantly, while this example serves to illustrate our framework, the contents of this section are novel, representing a natural extension of the principles discussed in the authors' previous work\cite{Ramos2023}.

\begin{figure}[htp]
	\centering
	\includegraphics[width=8.5cm]{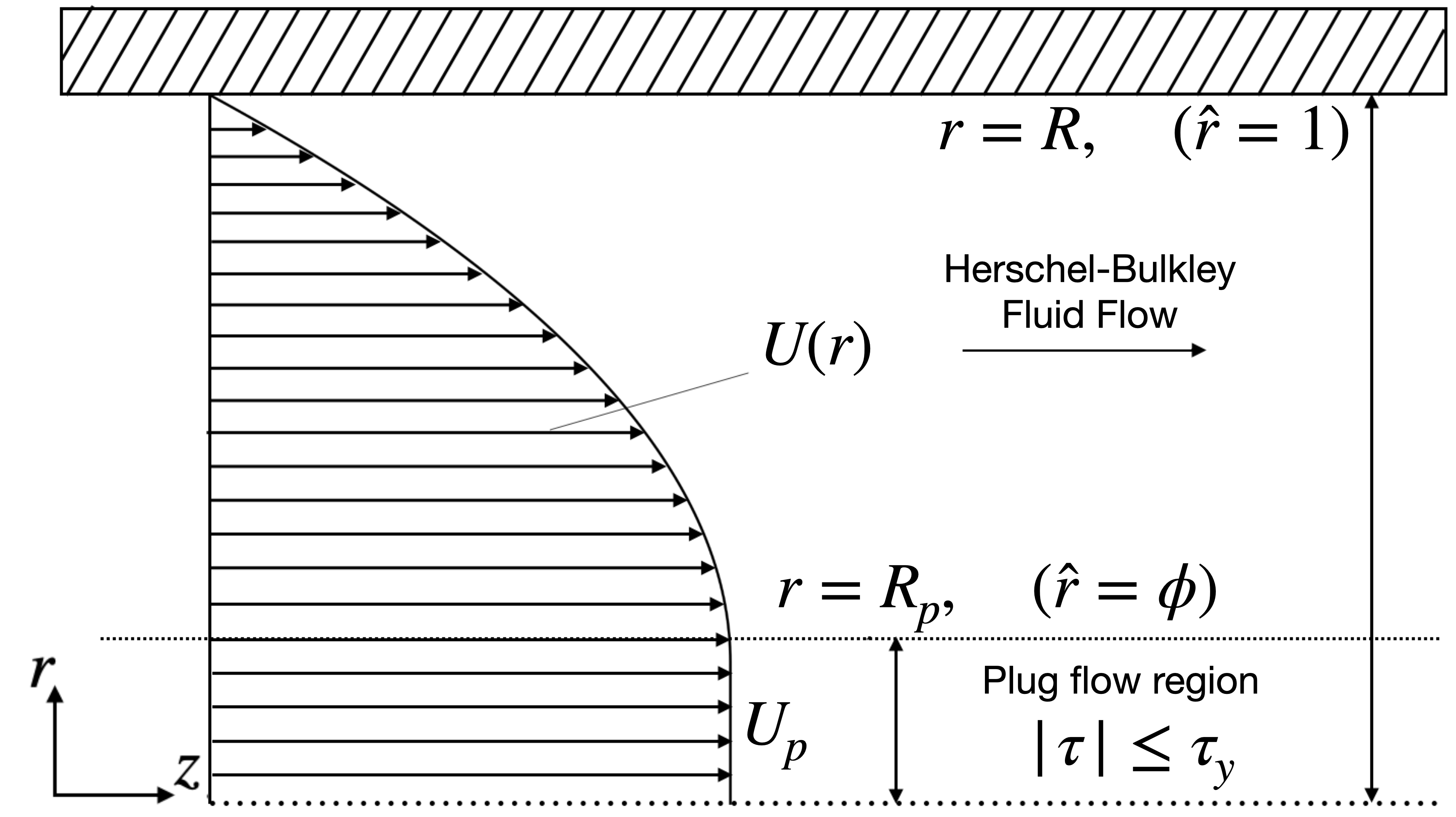}
	\caption{Mean velocity profile of laminar Herschel-Bulkley fluids.}
	\label{fig:HB_flow}
\end{figure}

A special geometric feature of laminar Herschel-Bulkley fluid flows is the presence of a solid plug-like core in the central region of the pipe, where the shear stress satisfies $\tau_{rz}<\tau_0$. The radius of the plug, $R_P$, is a function of the yield stress, $\tau_0$, the shear stress on the wall, $\tau_w$, and the radius of the pipe. In fact, because $\tau_{rz}(r)=-\frac{\partial P}{\partial z}\frac{r}{2}$, at $r=R_P$, the interface with the plug region, the shear stress satisfies $\tau_0=\tau_{rz}(R_P)=\frac{-\partial P}{\partial z}\frac{R_P}{2}$. Because $\tau_w=\tau_{rz}(R)=-\frac{\partial P}{\partial z}\frac{R}{2}$, where $R=D/2$ is the radius of the pipe, it follows that $\tau_0/\tau_w=R_P/R$. Fig. \ref{fig:HB_flow} illustrates a laminar Herschel-Bulkley fluid flow.

Analysis of Herschel-Bulkley fluid flows driven by a pressure gradient depends on the rheological parameters of the model $(K, \tau_0, n)$, in addition to fluid density $\rho$ and pipe diameter $D$. For such flows, their dynamics can be categorized according to the prescribed operational parameters: If the flow is defined by the mass flow rate, it is said to be parameterized in bulk velocity coordinates; conversely, if the flow is characterized by the pressure gradient or the wall-shear stress, it is described as being parameterized in friction coordinates.

The averaged bulk velocity is defined as $\bar{U}=Q/\left( \pi R^2 \right)$, where $Q=\int_0^R 2\pi rU\, dr$ is the volumetric flow rate of a pipe flow. In an incompressible context, the mass flow rate per unit volume is simply $\rho \bar{U}$. 
We now define some important dimensionless parameters. Let us start with the bulk Reynolds number, $Re_{MR}$, where MR stands for Metzner-Reed \cite{Metzner1955FlowON}:
\begin{equation}
	Re_{MR}=\frac{\rho \bar{U} D}{\mu_{eff}},
\end{equation}
where $\mu_{eff}=K\left(\frac{8\bar{U}}{D}\right)^{n-1}$
is the effective bulk viscosity.
Alternatively, a friction Reynolds number, denoted as $Re_\tau$, can be defined as
\begin{equation}
	Re_{\tau}:=\frac{\rho u_{\tau}D}{\mu_\tau},
\end{equation}
where 
	$\mu_\tau:=K^{1/n}\tau_w^{1-1/n}$
is the friction viscosity, and $u_\tau$ is defined in the same way as in Newtonian fluid flows. $Re_\tau$ can also be rewriten as $Re_\tau=D/\delta_{\nu}$, where $\delta_\nu:=\mu_\tau/(\rho u_\tau)$. This parameter diverges from the former formulation in that it does not rely explicitly on the flow rate. Instead, $Re_\tau$ is dependent upon the wall-shear stress, which correlates directly with the pressure gradient within the system. Another important dimensionless parameter in this context is the Hedstrom number $He$, defined as

\begin{equation}
	He:=\frac{\rho D^2}{K}\left(\frac{\tau_0}{K}\right)^{\frac{2-n}{n}}.
	\label{eq:hedstrom_definition}
\end{equation}

\begin{figure*}[htp]
	\centering
	\includegraphics[width=20cm]{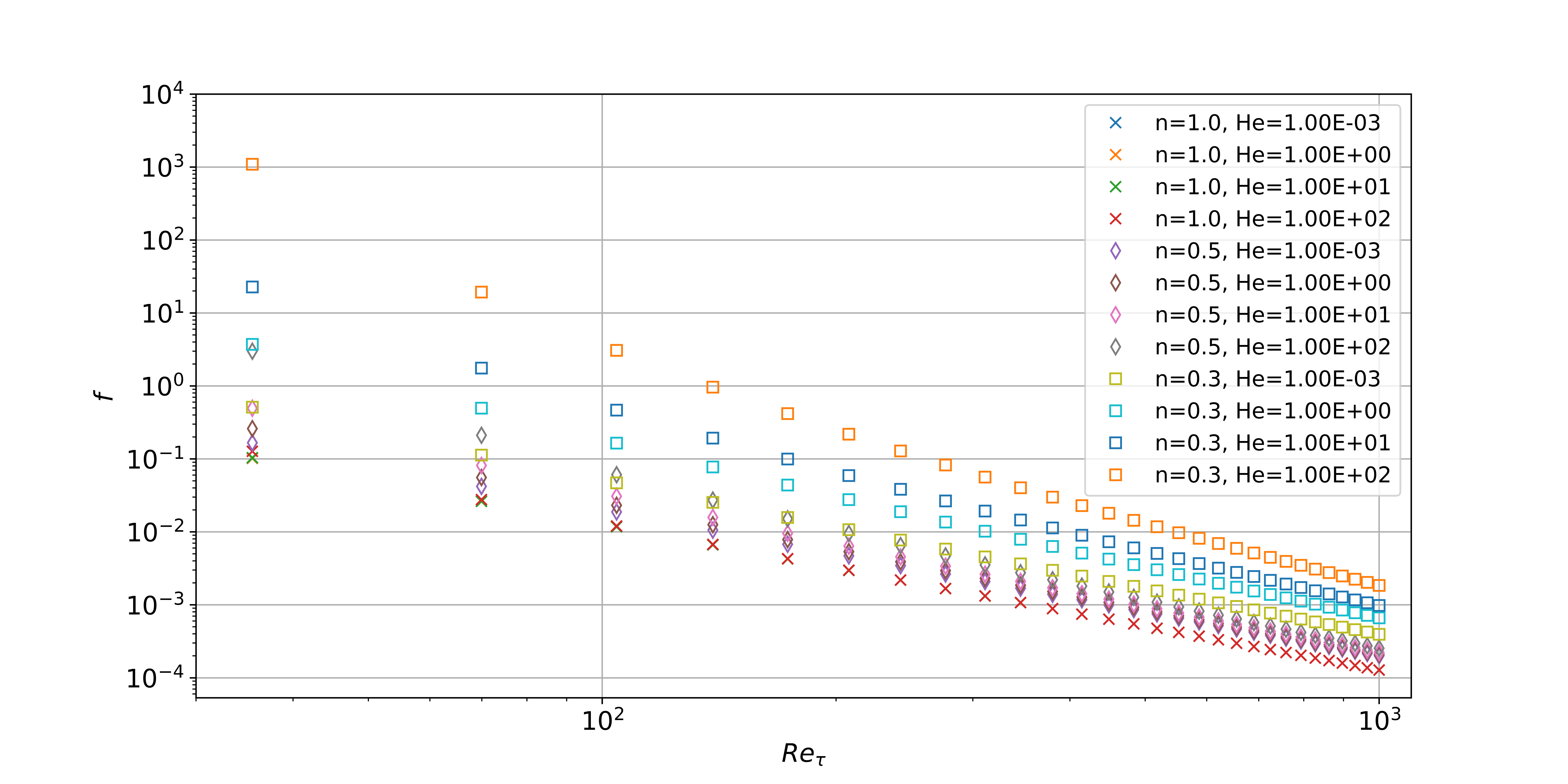}
	\caption{Friction factor of laminar Bingham fluid flows in friction coordinates. For each value of $n$ and $He$ fixed, we chose 20 points for $Re_{\tau}$ evenly distributed between $1$ and $10^2$}
	\label{fig:Buck_HB}
\end{figure*}
  In the literature, it is common to parametrize the flow with the bulk dimensionless coordinates $(He,Re)$, instead of the friction dimensionless coordinates $(He,Re_{\tau})$, see e.g.\cite{CR}. This choice is natural if one prescribes conditions on the mass flow rate instead of using the pressure gradient. However, in many applications, this choice is arbitrary, and the pressure gradient may be easier to determine through the use of manometers.

Lastly, we introduce $\phi$, a crucial dimensionless parameter, which is central to the discussions in this work. It is defined as follows:
\begin{equation}
	\phi:=\frac{R_P}{R}=\frac{\tau_0}{\tau_w}=\left(\frac{He}{Re_{\tau}^2}\right)^{\frac{n}{2-n}}.
	\label{phi_def}
\end{equation}
By its definition, $\phi$ inherently satisfies $0<\phi<1$. In the limiting case where $\phi\to 0$, the yield stress becomes negligible, resulting in a flow behavior similar to that of a power law fluid, characterized by the absence of a plug region. In contrast, as $\phi$ approaches 1, the flow ceases, leading to a scenario in which the plug extends throughout the pipe. A straightforward computation reveals that
\begin{equation}
	\phi = \mathcal{P}(n, He, Re_\tau) := \left( \frac{He}{Re_{\tau}^2} \right)^{\frac{n}{2-n}}.
	\label{eq:phi_property}
\end{equation}
This is an instance of incomplete similarity in relation to the parameters \( He \) and \( Re_\tau \).

From Buckingham's $\Pi-$Theorem, the pressure gradient satisfies a relation of the form: $-\frac{\partial P}{\partial z}=\frac{\rho \bar{U}^2}{D}\,f$, where $f$ is the so-called friction factor, a function of $(Re_{MR},He)$ or $(Re_{\tau},He)$. 
The fact that the friction factor can be written in bulk velocity coordinates, i.e. as $f=\mathcal{H}\left(Re_{MR},He\right)$, is related to the invariance of the friction factor by the action of the following similarity group:
\begin{eqnarray}
	\label{eq:sim_group_bulk_HB}
	&&K^*=A_1 K,\quad \rho^*=A_2 \rho,\quad D^*=A_3 D,\\
	&&\tau_0^*=\left(\frac{A_1^2}{A_2^n A_3^{2n}}\right)^{\frac{1}{2-n}}\tau_0,\nonumber \\ &&\bar{U}^*=\left(\frac{A_1}{A_2 A_3^n}\right)^{\frac{1}{2-n}}\bar{U},\nonumber \\ &&f^*=f.\nonumber
\end{eqnarray}
where $A_1, A_2, A_3$ are positive real numbers. The group depicted above is initially derived by scaling the dimensionally independent parameters $K$, $\rho$, and $D$ using arbitrary positive constants $A_1$, $A_2$, and $A_3$. Following this, we compute the appropriate scalings for the dimensionally dependent parameters $f$, $\tau_0$, and $\bar{U}$ to ensure the constancy of $Re_{MR}$, $He$, and the relationship $\mathcal{H}$. The importance of this similarity group emerges from the transformation \eqref{eq:sim_group_bulk_HB}, which guarantees $Re_{MR}^* = Re_{MR}$, $He^* = He$, and, more importantly, $f^* = \mathcal{H}\left(Re_{MR}^*, He^*\right)$. On the other hand, the fact that the friction factor can also be written in friction coordinates, as $f=\mathcal{L}(Re_{\tau},He)$, is related to the invariance of the friction by the action of the following similarity group:
\begin{eqnarray}
	\label{eq:sim_group_friction}
	K^*=A_1 K,\quad \rho^*&&=A_2 \rho,\quad D^*=A_3 D,\\
	\tau_0^*&&=\left(\frac{A_1^2}{A_2^n A_3^{2n}}\right)^{\frac{1}{2-n}}\tau_0,\nonumber \\ 
	\left(\frac{\partial P}{\partial z}\right)^*&&=\left(\frac{A_1^2}{A_2^n A_3^{2+n}}\right)^{\frac{1}{2-n}} \left(\frac{\partial P}{\partial z}\right),\nonumber \\ f^*&&=f.\nonumber
\end{eqnarray}
where $A_1, A_2, A_3$ are positive real numbers, and the deriving of such group follows a similar reasoning as above. 

For Bingham plastic fluids, the friction factor is determined implicitly in bulk velocity coordinates, as detailed in the Buckingham-Reiner equation \cite{CR}:
\begin{equation}
	\label{eq:BuckReiner}
	f=\frac{16}{Re_{MR}}\left[1+\frac{1}{6}\frac{He}{Re_{MR}}-\frac{1}{3}\frac{He^4}{f^3Re_{MR}^7}\right].
\end{equation}
To the best of the authors' knowledge, for Herschel-Bulkley fluids, there is no established relationship, explicit or implicit, that depends exclusively on either bulk velocity coordinates or friction coordinates. 

Hanks \cite{hanks} introduced a composite formula that intertwines bulk and friction coordinates: \( f = 16/(\psi(\phi, n)Re_{MR}  )\), where \( \psi \) is a function influenced by \( n \) and \( \phi \), and \( \phi \) is directly affected by \( \tau_w \). This complex interrelation complicates the calculations when transitioning between friction and bulk velocity coordinates due to the requirement of knowing both \( \tau_w \) and \( \bar{U} \) simultaneously. Specifically, given that \( f=2\tau_w/\rho\bar{U}^2 \) and \( \phi=\tau_0/\tau_w \), the Hanks equation emerges as implicit, lacking explicit similarity groups beyond those identified by dimensional analysis. 

We now focus on establishing the incomplete similarity relations for both the friction factor and the velocity profile of laminar Herschel-Bulkley fluids. Our approach begins with the derivation of an expression for the flow rate, which is expressed in terms of the wall shear stress and other pertinent parameters of the problem.

Initially, since the radial shear stress is $\tau_{rz} = \left(-\frac{\partial P}{\partial z}\right)\frac{r}{2}$ and the velocity is zero on the wall of the pipe, we can integrate the Herschel-Bulkey model with respect to the radial position, $r$. This leads to 
\begin{equation}
	\label{MVP_HB}
	U(r)=\frac{nD}{2(n+1)}\left(\frac{\tau_w}{K}\right)^{1/n}\left[(1-\phi)^{\frac{n+1}{n}}-(\frac{r}{R}-\phi)^{\frac{n+1}{n}}\right],
\end{equation}
in the fluid-like region, $R_p\leq r \leq R$. The corresponding MVP in the plug region can be obtained by substituting $r=R_p$ in 
equation above. This yields
\begin{equation}
	\label{MVP-plug_HB}
	U_p(r)=\frac{nD}{2(n+1)}\left(\frac{\tau_w}{K}\right)^{1/n}(1-\phi)^{\frac{n+1}{n}},
\end{equation}
in the plug-region, $0\leq r\leq R_p$. After integrating the velocity profiles over the cross-sectional area of the pipe, we arrive at the expression for the mean velocity, $\bar{U}$, given by
\begin{equation}
	\label{barU_HB}
	\bar{U}=\frac{nD}{2}\left(\frac{\tau_w}{K}\right)J(\phi,n),
\end{equation}
where
\begin{equation}
	J(\phi,n)=(1-\phi)^{\frac{n+1}{n}}\left[\frac{(1-\phi)^2}{3n+1}+\frac{2\phi(1-\phi)}{2n+1}+\frac{\phi^2}{n+1}\right].
\end{equation}
Therefore, we obtain:
\begin{equation}
	\label{eq:BuckReiner_HB}
	f=\Phi(n, He, Re_\tau)=\frac{\mathcal{M}\left(\frac{He}{Re_\tau^2},n\right)}{Re_\tau^2},
\end{equation}
where $\mathcal{M}(\phi,n)=8/(n^2J^2(n,\phi))$.

\begin{figure*}[htp]
	\centering
	\includegraphics[width=18cm]{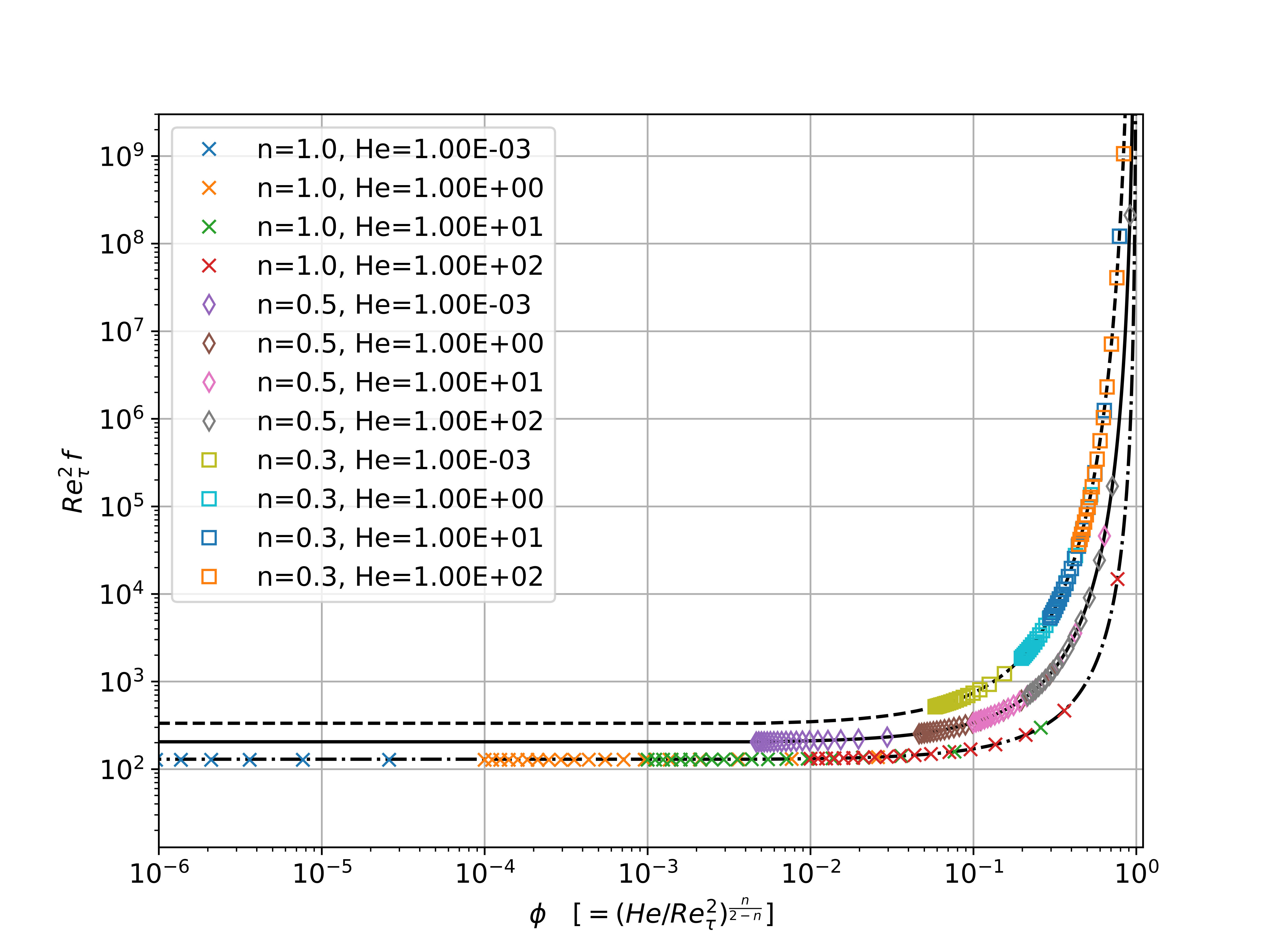}
	\caption{Friction factor of laminar Bingham fluid flows in friction coordinates. For each value of $n$ and $He$ fixed, we chose 20 points for $Re_{\tau}$ evenly distributed between $1$ and $10^2$.}
	\label{fig:collapse_incomplete_HB}
\end{figure*}

Figure \ref{fig:Buck_HB} presents twelve distinct datasets, each representing laminar friction curves plotted against the friction Reynolds number $Re_\tau$. These curves vary according to the non-dimensional parameters $He$, $Re_{\tau}$, and $n$. In Figure \ref{fig:collapse_incomplete_HB}, the concept of incomplete similarity is visualized by combining friction factor data from these twelve sets into three unique similarity curves. This demonstration of incomplete similarity is associated with specific invariance properties, which are detailed through renormalization group operations as follows:

\begin{eqnarray}
	\label{symfric}
	K^*&&=K,\quad \rho^*=\rho,\quad D^*=D\\
	\tau_0^*&&=B_1 \tau_0,\quad \left(\frac{\partial P}{\partial z}\right)^*=B_1\left(\frac{\partial P}{\partial z}\right),\nonumber\\
	f^*&&=B_1^{\frac{n-2}{n}}f,\nonumber
\end{eqnarray}
where $B_1$ is a positive real number. We remark that this symmetry group cannot be obtained through pure dimensional reasoning and that there is no similar invariance relation in purely bulk velocity coordinates.

We now extend our analysis to the mean velocity profile. Let $\hat{r}=r/R$. By Buckingham's $\Pi-$Theorem, the velocity profile can be rewritten either in bulk velocity coordinates, as $U=\mathcal{U}(r;D,\rho,\mu,\tau_0,\bar{U})=\bar{U}\Psi(\hat{r},He,Re)$, which is related to the following Buckingham's similarity group:
\begin{align}
	\label{eq:sim_group_bulk_2_HB}
	K^*&=A_1 K,\; \rho^*=A_2 \rho,\; D^*=A_3 D, \; r^*=A_3r\\
	\tau_0^*&=\left(\frac{A_1^2}{A_2^n A_3^{2n}}\right)^{\frac{1}{2-n}}\tau_0,\; \bar{U}^*=\left(\frac{A_1}{A_2 A_3^n}\right)^{\frac{1}{2-n}}\bar{U}, \nonumber \\ 
	U^*&=\left(\frac{A_1}{A_2 A_3^n}\right)^{\frac{1}{2-n}}U.\nonumber
\end{align}
or in friction coordinates, as $U=\mathcal{U}(r;D,\rho,\mu,\tau_0,\frac{\partial P}{\partial z})=u_{\tau}\Phi(\hat{r},He,Re_{\tau})$, which is related to the symmetry:
\begin{eqnarray}
	\label{eq:sim_group_bulk_3_HB}
	&&K^*=A_1 K,\; \rho^*=A_2 \rho,\; D^*=A_3 D, \; r^*=A_3r\\
	&&\tau_0^*=\left(\frac{A_1^2}{A_2^n A_3^{2n}}\right)^{\frac{1}{2-n}}\tau_0,\;\nonumber\\
	&&\left(\frac{\partial P}{\partial z}\right)^*=\left(\frac{A_1^2}{A_2^n A_3^{2+n}}\right)^{\frac{1}{2-n}}\left(\frac{\partial P}{\partial z}\right),\nonumber\\
	&&U^*=\left(\frac{A_1}{A_2 A_3^n}\right)^{\frac{1}{2-n}}U.\nonumber
\end{eqnarray}

\begin{figure*}[htp]
	\centering
	\includegraphics[width=12cm]{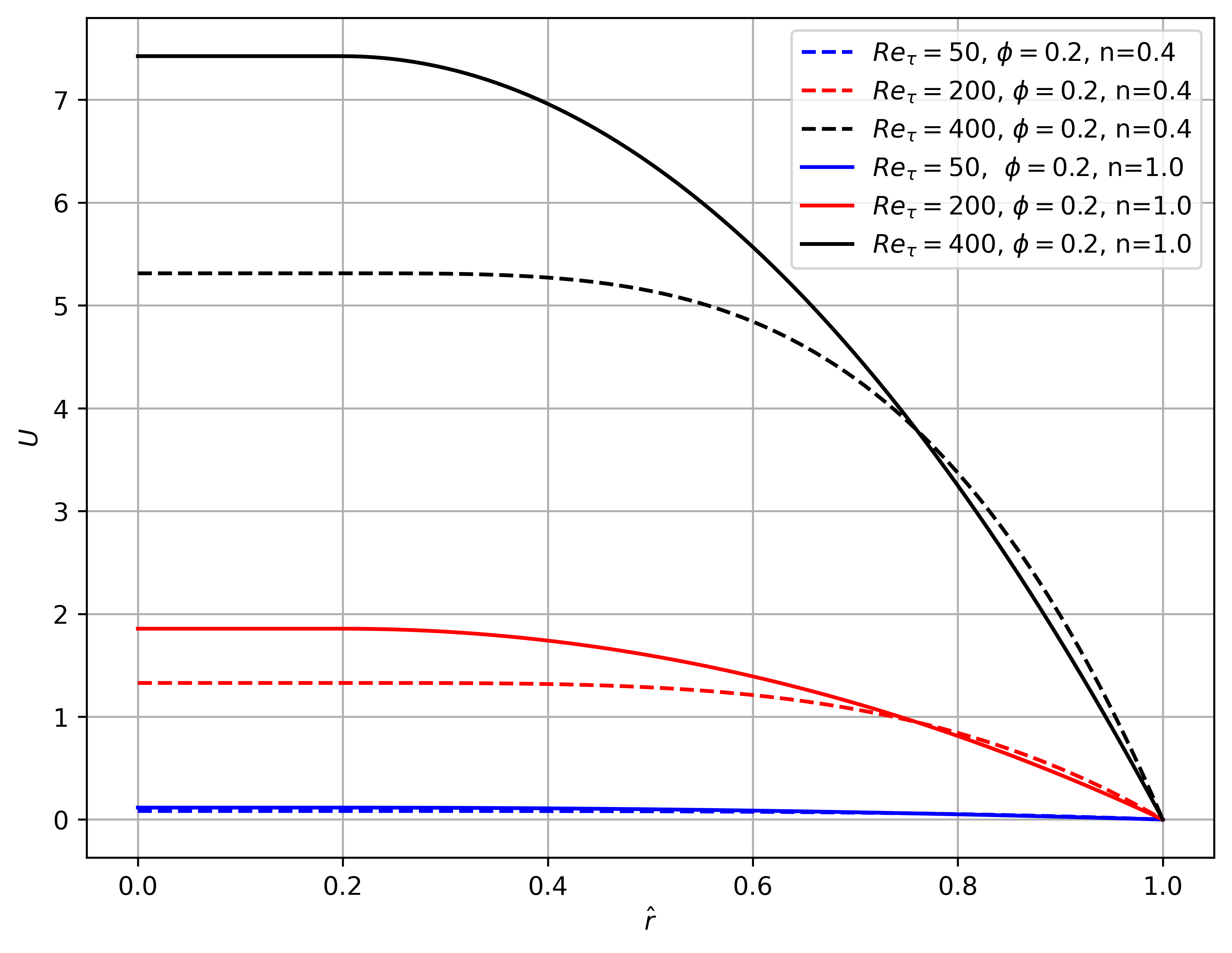}
	\caption{Velocity profiles for $\phi=0.2$, $n=0.4$ (dashed line), $n=1.0$ (solid line), $Re_\tau=50$ (blue), $Re_\tau=100$ (red), $Re_\tau=400$ (black). }
	\label{fig:vel_profiles_HB}
\end{figure*}

In friction coordinates, an explicit formula can be obtained by simple manipulation of Equation \eqref{MVP_HB}:
\begin{equation} \label{eq:HB_mvp_noplug}
	\frac{U}{u_{\tau}}= Re_{\tau}\frac{1}{4}\left[(1-\phi)^{\frac{n+1}{n}}-(\hat{r}-\phi)^{\frac{n+1}{n}}\right],
\end{equation}
for $\phi\leq \hat{r}\leq 1$, and
\begin{equation} \label{eq:HB_mvp_plug}
	\frac{U_p}{u_{\tau}}=Re_{\tau}\frac{1}{4}(1-\phi)^{\frac{n+1}{n}}
\end{equation}
for $0\leq \hat{r}\leq \phi$. This is a statement of incomplete similarity in friction coordinates in the sense that:
\begin{equation} \label{eq:incomplete_similarity_HB_equation}
	\frac{U}{u_{\tau}}=Re_{\tau}\Phi(\hat{r},n,He,Re_{\tau})=\Phi^{(1)}\left(\hat{r},n,\frac{He}{Re_\tau^2}\right),    
\end{equation}
for $\phi\leq \hat{r}\leq 1$,
and
\begin{equation}
	\frac{U}{u_{\tau}}=\Phi_p(n,He,Re_{\tau})= Re_{\tau}\Phi_p^{(1)}\left(n,\frac{He}{Re_\tau^2}\right),    
\end{equation}
for $0\leq \hat{r}\leq \phi$, where 

\begin{equation}
	\Phi^{(1)}\left(\hat{r},n,\phi\right)=\frac{1}{4}\left[(1-\phi)^{\frac{n+1}{n}}-(\hat{r}-\phi)^{\frac{n+1}{n}}\right],
\end{equation}
and $\Phi_p(n,He,Re_{\tau})=\Phi^{(1)}\left(\phi,n,\phi\right)$.

This is related to the scaling of the mean velocity profile by the action of the following renormalization group:
\begin{eqnarray}
	\label{RenGroupFriction_HB}
	\mu^*&&=\mu,\; \rho^*=\rho,\; D^*=D,\;r^*=r\\
	\tau_0^*&&=B_1 \tau_0,\; \left(\frac{\partial P}{\partial z}\right)^*=B_1\left(\frac{\partial P}{\partial z}\right),\; U^*=B_1^{\frac{1}{n}} U,\nonumber
\end{eqnarray}
where $B_1$ is a positive real number.

\begin{figure*}[htp]
	\centering
	\includegraphics[width=12cm]{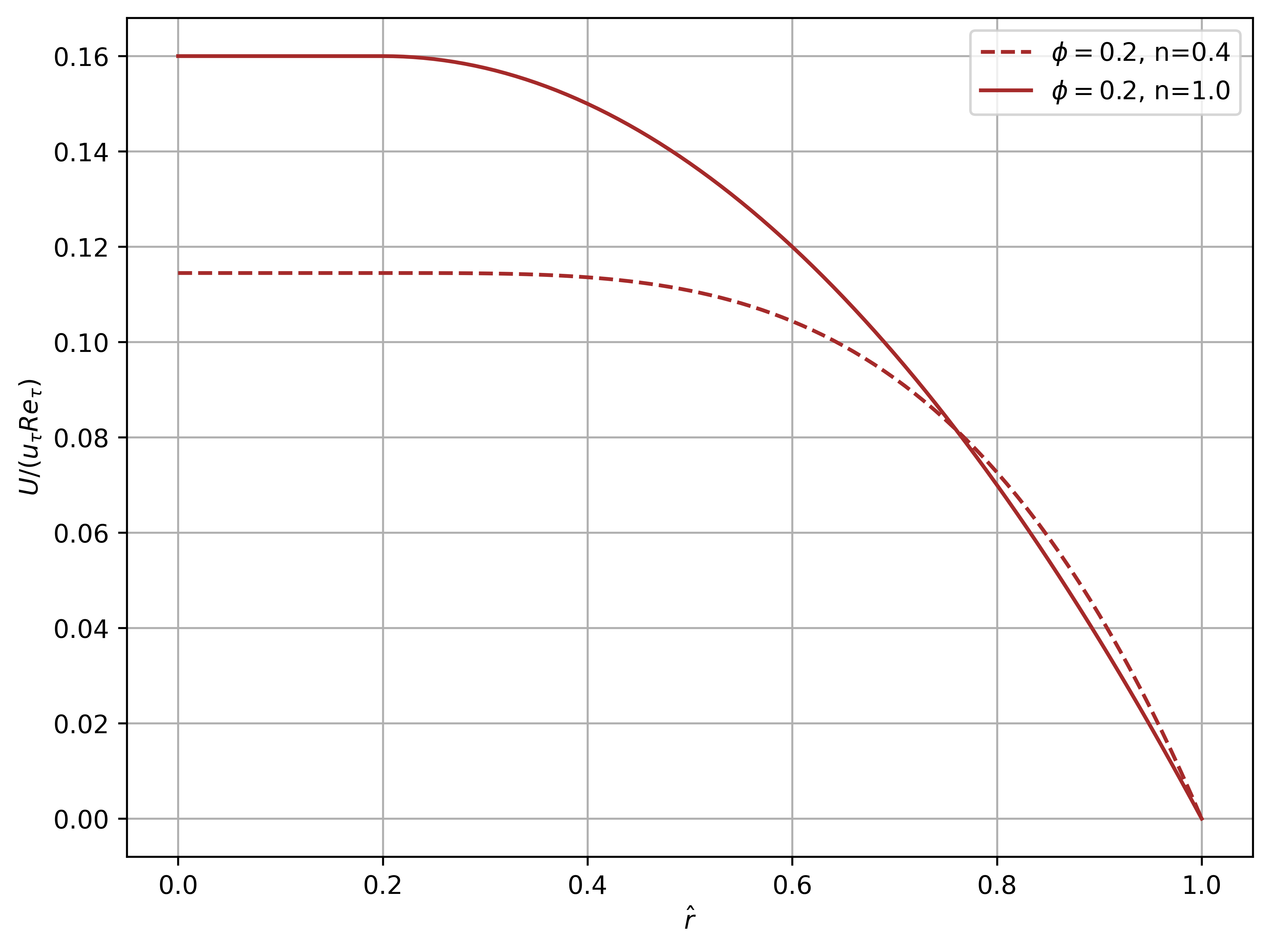}
	\caption{Collapse of velocity profiles in Figure \ref{fig:vel_profiles_HB} for different values of $Re_\tau$. $\phi=0.2$, $n=0.4$ (dashed line), $n=1.0$ (solid line).}
	\label{fig:collapse_incomplete_VP}
\end{figure*}

To elucidate the scaling phenomenon, we present an example inspired by the studies in \cite{Ramos2023, SWAMEE2011178}. Consider six distinct configurations of Herschel-Bulkley fluids characterized by a viscosity of $\mu=0.035, Pa.s$, a density of $\rho=1200, Kg/m^3$, and an effective pipe diameter of $D=0.1,m$. The friction Reynolds number $Re_\tau$, defined as $(\rho D/\mu)u_\tau$ where $u_\tau$, allows the selection of various pressure gradients $\partial P/\partial z$ to achieve $Re_{\tau}$ values within the set ${10,100,200,400}$. Furthermore, by varying the yield stress parameter $\tau_0$, we maintain a consistent ratio $\phi=\tau_0/\tau_w=0.2$. Using the relationship $u_\tau=Re_\tau \mu D/\rho$, we plot the velocity profiles $U$ and the normalized velocities $U/U_\tau Re_\tau$ for the fluid behavior indices $n=0.4$ and $n=1$, against the normalized distance $\hat{y}$, in Figures \ref{fig:vel_profiles_HB} and \ref{fig:collapse_incomplete_VP}. Remarkably, the six distinct velocity profiles in Figure \ref{fig:vel_profiles_HB} are unified into two curves in Figure \ref{fig:collapse_incomplete_VP} by applying the renormalization techniques discussed previously.

It is crucial to note that the introduction of yield stress disrupts the symmetry observed in laminar velocity profiles of simple power-law fluids, which typically demonstrate complete similarity in bulk velocity coordinates. Therefore, within the viscoplastic regime, the complete similarity condition cannot be fully applied, making the incomplete similarity framework the natural approach.

To conclude this section, we illustrate the use of Barenet to obtain the velocity profile of laminar Herschel-Bulkley flows discussed earlier. We generate data according to the given equations for 100 different values of both \(He\) and \(Re_\tau\) in the interval \([10, 100]\), and for three different values of \(n\) (namely \(n=0.3\), \(0.5\), and \(1.0\)). It was established that we should find the same exponents for incomplete similarity regardless of the value of \(n\). Specifically, we should find similarity exponents as in equation \eqref{eq:incomplete_similarity_HB_equation}, i.e.:
\begin{equation}
    u^+ = Re_\tau^{-\xi_3} \Phi^{(1)}_n \left( \hat{r} Re_\tau^{\xi_3^{(1)}}, He Re_\tau^{\xi_3^{(2)}} \right).
\end{equation}
With $\xi_3 = -1$, $\xi_3^{(1)} = 0$ and $\xi_3^{(2)} = -2$. By using the our network once again and after training for 10 epochs for each value of $n$ due to the large availability of generated data, the estimates of the exponents can be found in table \ref{table:HB_exponents}.

\begin{table}
\begin{center}
\begin{tabular}{ | m{5em} | m{1cm}| m{1.5cm} | m{1cm} | } 
  \hline 
  Exponents & $\hat{\xi}_3$ & $\hat{\xi}_3^{(1)}$ & $\hat{\xi}_3^{(2)}$ \\
  \hline
  $n=0.3$ & $-1.02$ & $1.6 \times 10^{-3}$ & $-1.97$ \\ 
  \hline
  $n=0.5$ & $-1.01$ & $6.7 \times 10^{-4}$ & $-2.05$ \\ 
  \hline
  $n=1.0$ & $-0.99$ & $5.8 \times 10^{-4}$ & $-2.26$ \\ 
  \hline
\end{tabular}
\end{center}
\caption{Incomplete similarity exponents found by the Barenet for laminar Hershcel-Bulkley data.}
\label{table:HB_exponents}
\end{table}

Plots of both raw data and renormalized collapsed data in the coordinates $\hat{r} Re_\tau^{\hat{\xi}_3^{(1)}} \times He Re_\tau^{\hat{\xi}_3^{(2)}} \times u^+ Re_\tau^{\hat{\xi}_3}$, can be found in figures \ref{fig:laminar_HB} and \ref{fig:laminar_HB_renormalized}.

\begin{figure*}[htp]
        \centering
	\includegraphics[width=17cm]{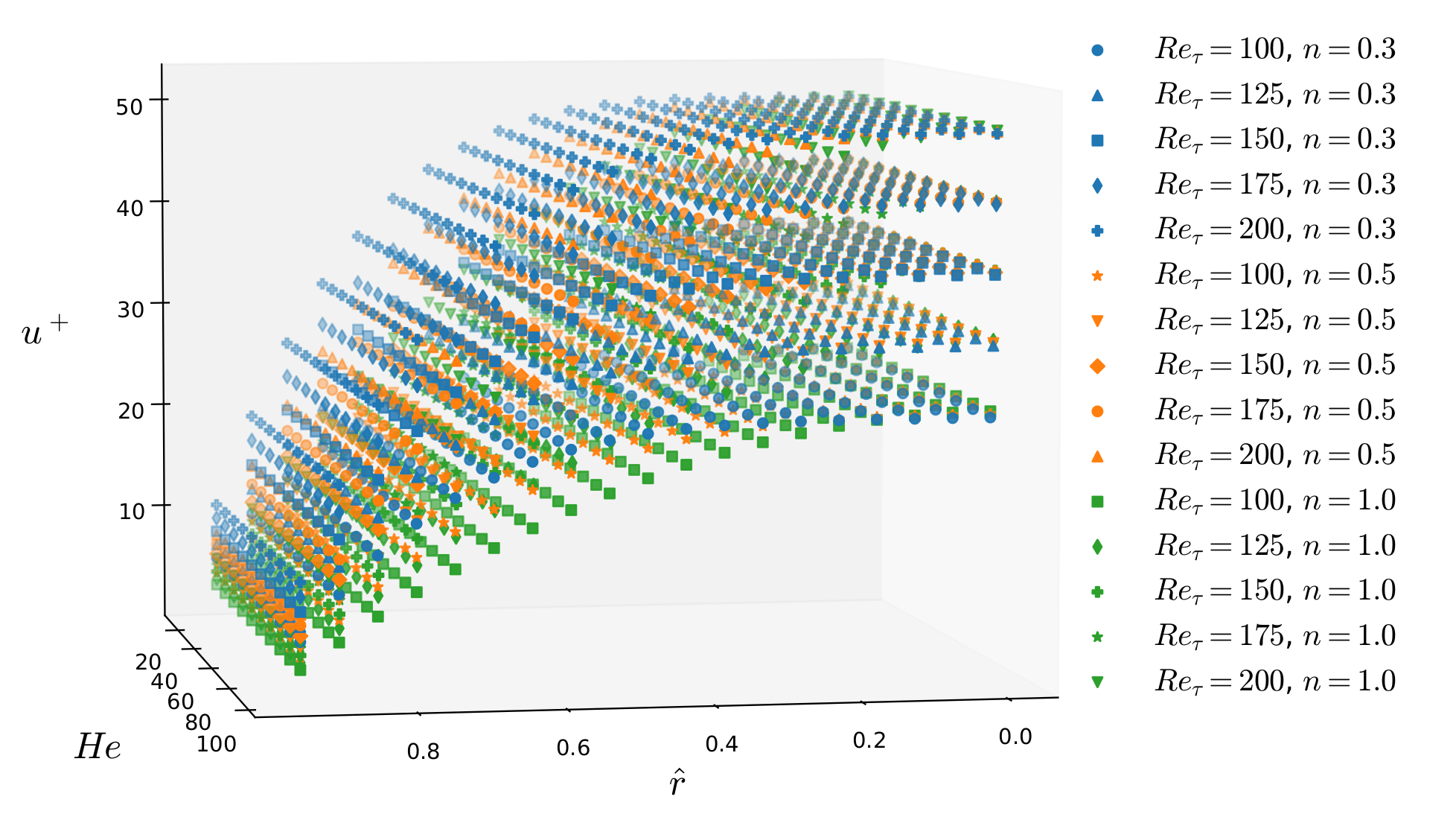}
	\caption{Laminar Herschel-Bulkley flow}
	\label{fig:laminar_HB}
\end{figure*}

\begin{figure*}[htp]
	\includegraphics[width=17cm]{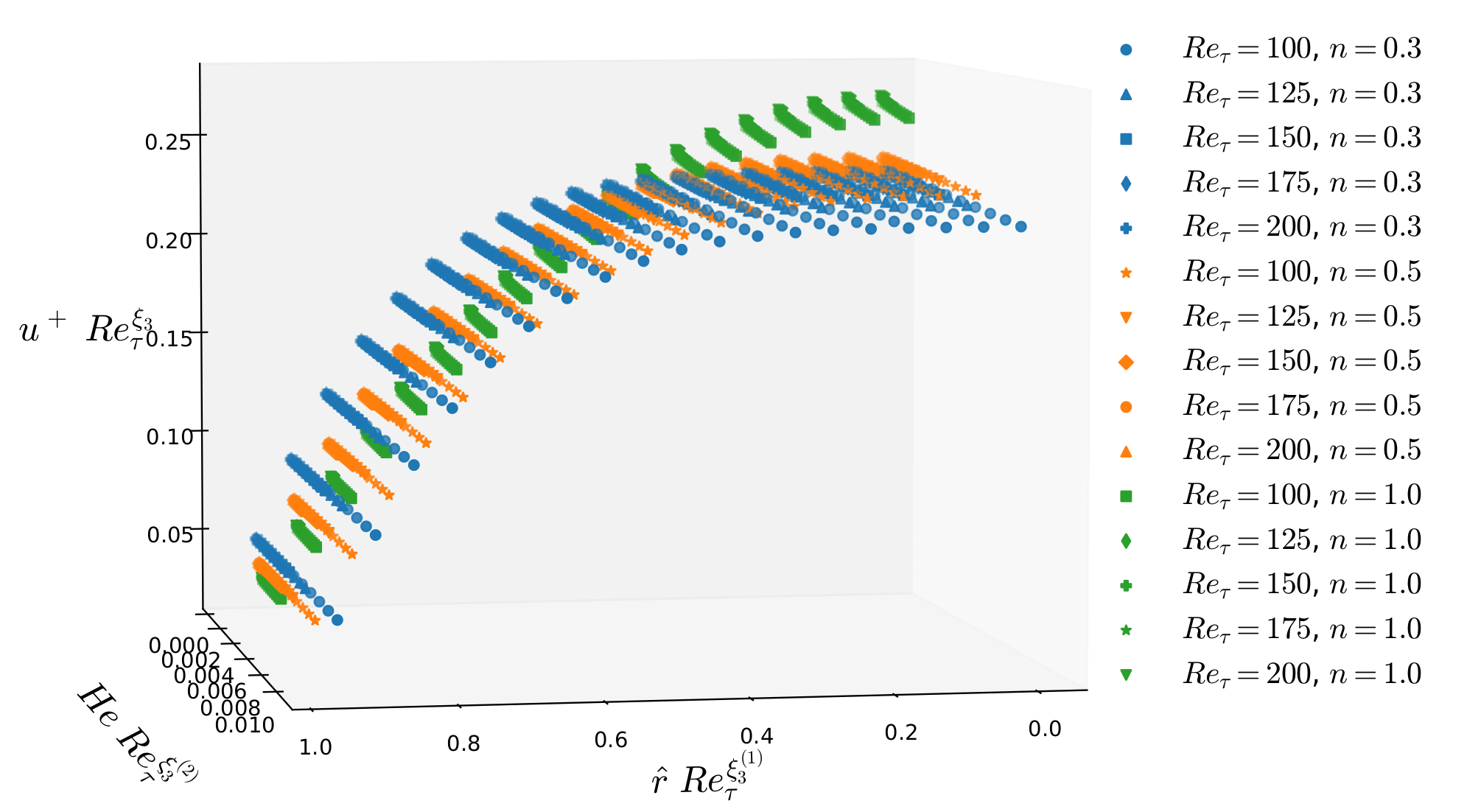}
	\centering
	\caption{Laminar Herschel-Bulkley flow after renormalization with Barenet exponents}
	\label{fig:laminar_HB_renormalized}
\end{figure*}

\section{Conclusions} \label{sec:conclusions}
In this work, we introduced a neural network algorithm designed to automatically identify similarity relationships from the data. This algorithm uncovers the underlying physical laws that relate dimensionless quantities to their dimensionless variables and coefficients. In addition, we developed a linear algebra framework to derive the symmetry groups associated with these similarity relations.

Our algorithm performs two primary tasks: implementing an automatic similarity group calculator and proposing a neural network architecture to identify incomplete similarity exponents. The effectiveness of this neural network was demonstrated through examples involving laminar Newtonian and non-Newtonian flows, as well as turbulent Newtonian flows in pipes.

Compared with other works related to data-driven dimensional analysis and similarity, this manuscript offers several significant contributions. We have developed a new dimensionless construction framework that generalizes previous ones, while also providing solid mathematical results that aid in calculating the similarity and renormalization groups associated with them. Our neural network approach for incomplete similarity discovery introduces new ideas to the field. For example, the Buckinet proposed by Bakarji et al. \cite{bakarji2022} includes a loss term that favors sparsity in the exponents of the dimensionless quantities' construction. This approach is ideal for those seeking the simplest possible construction, which should be more physically meaningful or at least more interpretable.

However, as discussed in this paper, Barenet is not primarily focused on simple dimensionless quantities. Instead, it aims to find which possible dimensionless quantities have the simplest dimensionless laws, where simplicity is measured by the number of arguments in such laws. We argue that this approach should be favored as it leads to the discovery of both complete and incomplete similarities and uncovers hidden similarity groups. Using the theorems in Section \ref{sec:generalized_similarity_groups}, we can maintain almost all physical interpretability, as we can always calculate the scaling effects in our phenomena, even if the dimensionless quantities found are complex.

There are several ways in which our network can be improved. First, Barenet currently considers only the cases where the entire domain of parameters obeys the same similarity law. For example, the log-law is valid only in the so-called log-layer. However, the notion of the log-layer is disputed, and the correct range's localization may affect the discovery of similarity exponents and associated symmetry groups. This limitation hinders the network's application to challenging tasks such as finding the proper scaling of spectral quantities in wall-bounded flows, which exhibit multi-scaling in different domains. Nevertheless, it is important to note that in the example of turbulent flows in rough pipes, we did not separate the pure laminar and Blasius ranges from the roughness-dominated range, yet the code correctly approximated Goldenfeld's law.

Another important aspect for improvement is that our network currently cannot handle expressions such as \(\Pi_i^{g(\Pi_j)}\). Therefore, expressions such as \(\phi = \left( \frac{He}{Re_{\tau}^2} \right)^{\frac{n}{2-n}}\) found in \eqref{eq:phi_property} cannot be explicitly derived from Barenet with the proper exponents in algebraic form. It can only be approximated as a black-box function for each \(n\). Moreover, even if our parameters are just present in the classical form of exponentiation and multiplication, we do not have a mathematical result which states that Barenblatt's definition of incomplete similarity is able to generate every possible renormalized dimensionless quantity. Much like the generalization of dimensionless construction, maybe a mathematical theory of generalized incomplete similarity should be developed in order to clarify this matter.

One final aspect is the potential integration of algorithms like Buckinet with Barenet to transition from dimensional data to the associated similarity scaling laws and their respective symmetry groups. Additionally, integrating Barenet and Buckinet with symbolic regression algorithms, such as SINDy, could create a powerful tool for both experimentalists and theoreticians. This enhanced package would assist experimentalists in their experimental design tasks and enable theoreticians to uncover symmetry relations directly from data.

These extensions are under investigation, and we intend to report on these developments in future work.

\section*{Acknowledgements}

The authors would like to thank Professors Roney Thompsom, Luca Moriconi, Cassio Oishi and Cesar Niche for their careful and insightful comments on the contents of this manuscript. This work was partially funded by CAPES Foundation and by ANP.

\appendix

\section{Multiple Dimensionally Dependent Parameters (MDDP) Dimensionless Construction}\label{sec:mddp}

It is not hard to notice that Barenblatt's dimensionless construction fails to encompass the definition of some dimensionless parameters present in the literature. An example already seen in this manuscript is the friction distance from the wall $y^+$ when expressed in terms of the fundamental parameters $D, \rho, \mu, -\frac{dP}{dz}, y$:
\begin{equation}
	y^+ \sim y \left( -\frac{dP}{dz} \right)^{1/2} D^{1/2} \rho^{1/2} \mu^{-1}.
\end{equation}
The problem with the above dimensionless quantity is that it makes use of two different dimensionally dependent parameters (namely $y$ and $dP/dz$), and this was not allowed in the dimensionless constructions presented in equation \eqref{eq:barenblatt_construction}. Based on these remarks, the development of a new theory was needed in order to make the correct calculations when constructing similarity and renormalization groups.

However, before giving our new definition, we will do some preparatory work. Let $\beta \in \R^l$ and $\alpha \in \R^m$ and define the quantity $d\left( \beta, \alpha \right)$ as follows:
\begin{equation}
    d \left( \beta, \alpha \right) = b_1^{\beta_1} \dots b_l^{\beta_l} a_1^{\alpha_1} \dots a_m^{\alpha_m}.
\end{equation}
\begin{lemma} \label{lemma:exponents_uniqueness}
    For each $\beta \in \R^l$ there is a unique $\alpha \in \R^m$ such that $d\left( \beta, \alpha \right)$ is dimensionless.
\end{lemma}
\begin{proof}
    First we will prove existence. Because each $b_j$ is a dimensionally dependent parameter, we know there exists $\alpha^{(j)} \in \R^m$ such that $b^{\beta_j} a_1^{\alpha_1^{(j)}} \dots a_m^{\alpha_m^{(j)}}$ is dimensionless. Define $\alpha = \sum_{j=1}^l \alpha^{(j)}$ and notice that:
    \begin{align}
        & d \left( \beta, \alpha \right) = b_1^{\beta_1} \dots b_l^{\beta_l} a_1^{\sum_{j=1}^l \alpha_1^{(j)}} \dots a_m^{\sum_{j=1}^l \alpha_m^{(j)}} = \nonumber \\ & \nonumber \\  & = \left( b^{\beta_1} a_1^{\alpha_1^{(1)}} \dots a_m^{\alpha_m^{(1)}} \right) \dots \left( b^{\beta_l} a_1^{\alpha_1^{(l)}} \dots a_m^{\alpha_m^{(l)}} \right).
    \end{align}
    The last quantity is a product of dimensionless quantities and, thus, is also dimensionless, and this establishes the existence of such $\alpha \in \R^m$. To prove uniqueness, suppose that there are $\alpha \neq \alpha' \in \R^m$ such that $d \left( \beta, \alpha \right)$ and $d \left( \beta, \alpha' \right)$ are dimensionless. This implies that $d \left( \beta, \alpha \right) / d \left( \beta, \alpha' \right)$ is also dimensionless, but notice that:
    \begin{equation}
        \frac{d \left( \beta, \alpha \right)}{d \left( \beta, \alpha' \right)} = \frac{b_1^{\beta_1} \dots b_l^{\beta_l} a_1^{\alpha_1} \dots a_m^{\alpha_m}}{b_1^{\beta_1} \dots b_l^{\beta_l} a_1^{\alpha_1'} \dots a_m^{\alpha_m'}} = a_1^{\alpha_1 - \alpha_1'} \dots a_m^{\alpha_m - \alpha_m'}.
    \end{equation}
    The above equation would imply that the $a_i$'s are not dimensionally independent, a clear contradiction.
\end{proof}
We will now define the set $E$ of exponents which make $d\left( \beta, \alpha \right)$ dimensionless, i.e.:
\begin{equation}
E:= \left\{ \left( \beta, \alpha \right) \in \R^{l+m} \ | \ d\left( \beta, \alpha \right) \textnormal{ is dimensionless} \right\}.
\end{equation}
\begin{lemma}
    $E$ is a vector space of dimension $l$.
\end{lemma}
\begin{proof}
    Because the proof that $E$ is a vector space is a straightforward application of linear algebra concepts, we will concentrate on proving that its dimension is $l$. Let $e_j$ be the $j$-th unit vector of $\R^l$. Using Lemma \ref{lemma:exponents_uniqueness}, we know that there exists $\alpha^{(j)} \in \R^m$ such that $d\left( e_j, \alpha^{(j)} \right)$ is dimensionless. Moreover, we know that $\left( e_1, \alpha^{(1)} \right), \dots, \left( e_l, \alpha^{(l)} \right)$ are linearly independent vectors. We will now prove that these vectors generate the space $E$. Let $\left( \beta, \alpha \right) \in E$, we know that $\beta = \sum_{j=1}^l \beta_j e_j$ and also that $\sum_{j=1}^l \beta_j \left( e_j, \alpha^{(j)} \right) \in E$ because $E$ is a vector space, but notice that:
    \begin{equation}
        \sum_{j=1}^l \beta_j \left( e_j, \alpha^{(j)} \right) = \left( \sum_{j=1}^l \beta_j e_j, \sum_{j=1}^l \beta_j \alpha^{(j)} \right) = \left( \beta, \sum_{j=1}^l \beta_j \alpha^{(j)} \right).
    \end{equation}
    Now, because of the uniqueness in Lemma \ref{lemma:exponents_uniqueness}, we must have $\sum_{j=1}^l \beta_j \alpha^{(j)} = \alpha$, which proves that the aforementioned vectors are indeed a basis for the vector space $E$.
\end{proof}

The MDDP dimensionless construction is defined as:
\begin{align}
    &\Pi_j = b_1^{\beta_1^{(j)}} \dots b_l^{\beta_l^{(j)}} a_1^{\alpha_1^{(j)}} \dots a_m^{\alpha_m^{(j)}}, \ \ \ \ \ \  j = 1,\dots, l; \nonumber \\ & \nonumber \\
    & \Pi = a^\beta b_1^{\beta_1} \dots b_l^{\beta_l} a_1^{\alpha_1} \dots a_m^{\alpha_m}. \label{eq:mddp_construction_appendix}
\end{align}
Where the vectors $\left( \beta^{(1)}, \alpha^{(1)} \right), \dots, \left( \beta^{(l)}, \alpha^{(l)} \right) \in E$, the exponents subspace that make the $\Pi_j$'s dimensionless. It is also important to ask that the real numbers $\beta, \beta_1, \dots, \beta_l, \alpha_1, \dots, \alpha_m$ be chosen in a way that makes $\Pi$ dimensionless.

Recall that the Buckingham $\Pi$-Theorem asks that the $\Pi_j$'s be independent in the sense of exponentiation and multiplication. As a corollary of the two results above, we have the following.
\begin{corollary} \label{cor:independent_pi_j}
    The $\Pi_j$'s are independent by exponentiation and multiplication if and only if the vectors $\beta^{(1)}, \dots, \beta^{(l)} \in \R^l$ are linearly independent.
\end{corollary}
It remains now to seek under what conditions the Buckingham similarity group exists in the above construction. 

\textit{Proof of Claim I:} Suppose we want to rescale the dimensionally independent parameters by arbitrary positive constants $A_1, \dots, A_m$:
\begin{equation}
    a_1^* = A_1 a_1; \ \ \ \ a_2^* = A_2 a_2; \ \ \dots \ \ ; a_m^* = A_m a_m,
\end{equation}
We want thus to find exponents:
\begin{align}
    b_1^* &= A_1^{\delta_1^{(1)}} \dots A_m^{\delta_m^{(1)}} b_1 = b_1\prod_{k=1}^m A_k^{\delta_k^{(1)}}; \nonumber \\ & \nonumber \vdots \\ b_l^* & = A_1^{\delta_1^{(l)}} \dots A_m^{\delta_m^{(l)}} b_l = b_l \prod_{k=1}^m A_k^{\delta_k^{(l)}}; \nonumber \\ & \nonumber \\
    a^* &= A_1^{\delta_1} \dots A_m^{\delta_m} a = a \prod_{k=1}^m A_k^{\delta_k}.
\end{align}
Such that:
\begin{align}
    \Pi_1^* &= \Pi_1; \ \ \ \ \ \Pi_2^* = \Pi_2; \ \ \ \ \ \dots \ \ \ \ \ \Pi_l^* = \Pi_l; \nonumber \\ & \nonumber \\
    \Pi^* &= \Pi.
\end{align}
Let us first focus on the equations for the $\Pi_j$'s. Notice that:
\begin{align}
    &\Pi_j^* = \Pi_j \ \ \Rightarrow \nonumber \\ & \nonumber \\ & \Rightarrow {b_1^*}^{\beta_1^{(j)}} \cdots {b_l^*}^{\beta_l^{(j)}} {a_1^*}^{\alpha_1^{(j)}} \cdots {a_m^*}^{\alpha_m^{(j)}} = b_1^{\beta_1^{(j)}} \cdots b_l^{\beta_l^{(j)}} a_1^{\alpha_1^{(j)}} \cdots a_m^{\alpha_m^{(j)}} \ \ \Rightarrow \nonumber \\ & \nonumber \\
    & \Rightarrow \ \ b_1^{\beta_1^{(j)}} \prod_{k=1}^m A_k^{\beta_1^{(j)} \delta_k^{(1)}} \cdots b_l^{\beta_l^{(j)}} \prod_{k=1}^m A_k^{\beta_l^{(j)} \delta_k^{(l)}} \cdot A_1^{\alpha_1^{(j)}} a_1^{\alpha_1^{(j)}} \cdots  A_m^{\alpha_m^{(j)}} a_m^{\alpha_m^{(j)}} = \nonumber \\ & \nonumber \\ &= b_1^{\beta_1^{(j)}} \cdots b_l^{\beta_l^{(j)}} a_1^{\alpha_1^{(j)}} \cdots a_m^{\alpha_m^{(j)}} \ \ \Rightarrow \nonumber \\ & \nonumber \\
    & \Rightarrow \ \ A_1^{\alpha_1^{(j)} + \sum_{i=1}^l \beta_i^{(j)} \delta_1^{(i)}} \cdots A_m^{\alpha_m^{(j)} + \sum_{i=1}^l \beta_i^{(j)} \delta_m^{(i)}} = 1, \ \ \ \ \ \ \ \forall \ 1 \leq j \leq l.
\end{align}
We want the last equation to hold for any positive choice of $A_1, \dots, A_m$. This implies that every exponent must be zero. Fix $k \in \left\{ 1, \dots, m \right\}$ and let's look only for the exponents of $A_k$:
\begin{equation}
    \alpha_k^{(j)} + \sum_{i=1}^l \beta_i^{(j)} \delta_k^{(i)} = 0, \ \ \ \ \forall \ 1 \leq j \leq l.
\end{equation}
This can be written as the following linear problems:
\begin{equation}
    \begin{bmatrix}
        \beta_1^{(1)} & \cdots & \beta_l^{(1)} \\
        \vdots & \ddots & \vdots \\
        \beta_1^{(l)} & \cdots & \beta_l^{(l)}
    \end{bmatrix}
    \begin{bmatrix}
        \delta_k^{(1)} \\
        \vdots \\
        \delta_k^{(l)}
    \end{bmatrix} =
    \begin{bmatrix}
        -\alpha_k^{(1)} \\
        \vdots \\
        -\alpha_k^{(l)}
    \end{bmatrix} \ \ \ \ \forall \ 1 \leq k \leq m.
\end{equation}
Thus, the Buckingham similarity group exists if and only if the $m$ linear systems depicted above have a solution. However, one can notice that if the construction of the $\Pi_j$'s is independent in the sense of exponentiation and multiplication, Corollary \ref{cor:independent_pi_j} tells us that the lines of the $l \times l$ matrix on the left are linearly independent, and thus the $m$ systems always have a solution.

It now remains to find the similarity group exponents for $a$. But assuming we already know them for $b_1, \dots, b_l$ by solving the linear systems just mentioned, and following a line of thought similar to the deduction of Buckingham's similarity group for Barenblatt's classical construction, we arrive at:
\small
\begin{equation}
    \delta_1 = - \left( \frac{\alpha_1 + \sum_{i=1}^l \delta_1^{(i)} \beta_i}{\beta} \right); \ \ \ \ \dots \ \ \ \ ;\delta_m = - \left( \frac{\alpha_m + \sum_{i=1}^l \delta_m^{(i)} \beta_i}{\beta} \right),
\end{equation}
\normalsize
and thus the Buckingham's Similarity Group was found.

Before moving on to deducing the renormalization group, we can make a small numerical remark. Notice that the $l \times l$ matrix of the $\beta_i^{(j)}$ in the previous subsection does not depend on $k = 1, \dots, m$. So, by defining the matrices below:
\begin{align}
        &\mathcal{B} = \begin{bmatrix}
        \beta_1^{(1)} & \cdots & \beta_l^{(1)} \\
        \vdots & \ddots & \vdots \\
        \beta_1^{(l)} & \cdots & \beta_l^{(l)}
    \end{bmatrix}; \ \ \ \ 
    \Delta = \begin{bmatrix}
        \delta_1^{(1)} & \cdots & \delta_m^{(1)} \\
        \vdots & \ddots & \vdots \\
        \delta_1^{(l)} & \cdots & \delta_m^{(l)}
    \end{bmatrix}; \\ & \\
& \ \ \ \ \ \ \ \ \ \ \ \ \ \ \ \ \ \mathcal{A} = \begin{bmatrix}
        -\alpha_1^{(1)} & \cdots & -\alpha_m^{(1)} \\
        \vdots & \ddots & \vdots \\
        -\alpha_1^{(l)} & \cdots & -\alpha_m^{(l)},
    \end{bmatrix},
\end{align}
and remembering that our main goal is to find the exponents in the matrix $\Delta$, we see we can indeed look at it as a problem of solving the linear matrix equation:
\begin{equation} \label{eq:buckingham_group_linear_system}
    \mathcal{B} \Delta = \mathcal{A}, 
\end{equation}
where $\Delta$ is the unknown matrix. But again, if we assume the $\Pi_j$'s are constructed to be independent by exponentiation and multiplication, we know the matrix $\mathcal{B}$ is invertible and, thus, the only numerical step needed to calculate the exponents in the Buckingham's Similarity Group is the inversion of the matrix $\mathcal{B}$. The solution is:
\begin{equation}
    \Delta = \mathcal{B}^{-1} \mathcal{A}.
\end{equation}

\textit{Proof of Claim II:} As for the renormalization group deduction, suppose we have found an incomplete similarity relation just like in equations \eqref{eq:incomplete_similarity_definition_introduction} and \eqref{eq:auxiliary_incomplete_similarity_definition}. As done before, we will fix the dimensionally independent parameters and rescale the last $l - n$ dimensionally dependent parameters in the following way:
\begin{equation}
    b_{n+1}^* = B_{n+1} b_{n+1}; \ \ \ \ b_{n+2}^* = B_{n+2} b_{n+2}; \ \ \dots \ \ ; b_{l}^* = B_{l} b_{l}.
\end{equation}

Much like before, our main goal is to find suitable exponents $\mu^{(i)}_j$ and $\mu_j$ to make the incomplete similarity relation above unchanged. The exponents are distributed as before, i.e.
\begin{align}
    b_1^* &= B_{n+1}^{\mu_{n+1}^{(1)}} \dots B_l^{\mu_l^{(1)}} b_1; \ \ \ \ \ \  \dots \ \ \ \ \ b_n^* = B_{n+1}^{\mu_{n+1}^{(n)}} \dots B_l^{\mu_l^{(n)}} b_n; \nonumber \\ & \nonumber \\
    a^* &= B_{n+1}^{\mu_{n+1}} \dots B_l^{\mu_l} a.
\end{align}
For each $j \in \left\{ 1, \dots, n\right\}$, we will make the calculations in order to make the $j$-th argument of the $\Phi^{(1)}$ function in equation \eqref{eq:incomplete_similarity_definition_introduction}, i.e.
\begin{equation} \label{eq:unchanged_incomplete_similarity_relation}
    \Pi_j^* \cdot \Pi_{n+1}^{*\xi_{n+1}^{(j)}} \cdots \Pi_{l}^{*\xi_{l}^{(j)}} = \Pi_j \cdot \Pi_{n+1}^{\xi_{n+1}^{(j)}} \cdots \Pi_{l}^{\xi_{l}^{(j)}}, \ \ \ \ \ \forall \ 1 \leq j \leq n.
\end{equation}
Let's start by trying to express $\Pi_j^*$ in terms of the scaling constants $B_{n+1}, \dots , B_l$ and $\Pi_j$:
\begin{align}
    & \Pi_j^* = b_1^{* \beta_1^{(j)}} \cdots b_l^{* \beta_l^{(j)}} a_1^{* \alpha_1^{(j)}} \cdots a_m^{* \alpha_m^{(j)}} =  \nonumber \\ & \nonumber \\ & = \left( B_{n+1}^{\mu_{n+1}^{(1)}} \cdots B_{l}^{\mu_{l}^{(1)}} b_1 \right)^{\beta_1^{(j)}} \cdots \left( B_{n+1}^{\mu_{n+1}^{(n)}} \cdots B_{l}^{\mu_{l}^{(n)}} b_n \right)^{\beta_n^{(j)}} \times \nonumber \\ & \nonumber \\ & \ \ \ \ \ \ \times \left( B_{n+1} b_{n+1} \right)^{\beta_{n+1}^{(j)}} \cdots \left( B_{l} b_{l} \right)^{\beta_{l}^{(j)}} a_1^{\alpha_1^{(j)}} \cdots a_m^{\alpha_m^{(j)}} = \nonumber \\ & \nonumber \\ & = B_{n+1}^{\left\langle \beta_{\left\{ 1, \dots n \right\}}^{(j)} \ , \ \mu_{n+1}^{(1, \dots, n)} \right\rangle + \beta_{n+1}^{(j)}} \cdots B_{l}^{\left\langle \beta_{\left\{ 1, \dots n \right\}}^{(j)} \ , \ \mu_{l}^{(1, \dots, n)} \right\rangle + \beta_{l}^{(j)}} \Pi_j.
\end{align}
where $\beta_{\left\{ 1, \dots n \right\}}^{(j)} := \left( \beta_1^{(j)}, \dots, \beta_n^{(j)} \right)$ and $\mu_k^{(1, \dots, n)} := \left( \mu_k^{(1)}, \dots, \mu_k^{(n)} \right)$ for every $k=n+1, \dots, l$. Recalling that we are fixing $j$, we can make similar calculations for $\Pi_k^{*\xi_k^{(j)}}$ for each $k \in \left\{ n+1, \dots, l \right\}$ to arrive at the following.
\begin{align}
    & \Pi_k^{*\xi_k^{(j)}} = B_{n+1}^{\xi_k^{(j)} \left( \left\langle \beta_{\left\{ 1, \dots, n \right\}}^{(k)} \ , \ \mu_{n+1}^{(1, \dots, n)} \right\rangle + \beta_{n+1}^{(k)} \right)} \cdots \nonumber \\ & \nonumber \\ & \cdots B_l^{\xi_k^{(j)} \left( \left\langle \beta_{\left\{ 1, \dots, n \right\}}^{(k)} \ , \ \mu_{l}^{(1, \dots, n)} \right\rangle + \beta_{l}^{(k)} \right)} \Pi_k^{\xi_k^{(j)}}.
\end{align}
We can now substitute $\Pi_j^*, \Pi_{n+1}^*, \dots, \Pi_l^*$ in equation \eqref{eq:unchanged_incomplete_similarity_relation} and make the proper computations to arrive at:
\small
\begin{align}
    & B_{n+1}^{\left\langle \beta_{\left\{ 1, \dots n \right\}}^{(j)} \ , \ \mu_{n+1}^{(1, \dots, n)} \right\rangle + \beta_{n+1}^{(j)} + \left\langle \sum_{k=n+1}^l \xi_k^{(j)} \beta_{\left\{1, \dots, n \right\}}^{(k)} \ , \ \mu_{n+1}^{\left( 1, \dots, n \right)} \right\rangle + \sum_{k=n+1}^l \xi_k^{(j)} \beta_{n+1}^{(k)}} \cdots \nonumber \\ & \nonumber \\
    &  \cdots B_{l}^{\left\langle \beta_{\left\{ 1, \dots n \right\}}^{(j)} \ , \ \mu_{l}^{(1, \dots, n)} \right\rangle + \beta_{l}^{(j)} + \left\langle \sum_{k=n+1}^l \xi_k^{(j)} \beta_{\left\{1, \dots, n \right\}}^{(k)} \ , \ \mu_{l}^{\left( 1, \dots, n \right)} \right\rangle + \sum_{k=n+1}^l \xi_k^{(j)} \beta_{l}^{(k)}} = \nonumber \\ & \nonumber \\ & = 1.
\end{align}
\normalsize
By following a line of thought similar to the Buckingham's Similarity Group exponents, we want the equality above to hold for every possible choice of $B_{n+1}, \dots B_l > 0$. We must thus conclude that every exponent above must be equal to $0$. Looking just at the exponents for $B_{n+1}$, for example, and varying $j \in \left\{ 1, \dots, n \right\}$, we obtain the following equations.
\begin{align}
    &\left\langle \beta_{\left\{ 1, \dots n \right\}}^{(j)} + \sum_{k=n+1}^l \xi_k^{(j)} \beta_{\left\{1, \dots, n \right\}}^{(k)} \ , \ \mu_{n+1}^{(1, \dots, n)} \right\rangle = \nonumber \\ & \nonumber \\ & = - \left( \beta_{n+1}^{(j)} + \sum_{k=n+1}^l \xi_k^{(j)} \beta_{n+1}^{(k)} \right).
\end{align}
The $n$ equations above translate into the following linear system:
\begin{align}
    & \begin{bmatrix}
    \text{---} & \beta_{\left\{ 1, \dots n \right\}}^{(1)} + \sum_{k=n+1}^l \xi_k^{(1)} \beta_{\left\{1, \dots, n \right\}}^{(k)} & \text{---} \\
    \vdots & \vdots & \vdots \\
    \text{---} & \beta_{\left\{ 1, \dots n \right\}}^{(n)} + \sum_{k=n+1}^l \xi_k^{(n)} \beta_{\left\{1, \dots, n \right\}}^{(k)} & \text{---}
    \end{bmatrix}
    \begin{bmatrix}
    \vrule \\ \ \\
    \mu_{n+1}^{\left( 1, \dots, n \right)} \\ \ \\
    \vrule
    \end{bmatrix} = \nonumber \\ & \nonumber \\ & = - 
    \begin{bmatrix}
    \beta_{n+1}^{(1)} + \sum_{k=n+1}^l \xi_k^{(1)} \beta_{n+1}^{(k)} \\ 
    \vdots \\
    \beta_{n+1}^{(n)} + \sum_{k=n+1}^l \xi_k^{(n)} \beta_{n+1}^{(k)}
    \end{bmatrix}.
\end{align}
We can, of course, generalize these computations in order to obtain a linear system for every $i \in \left\{ n + 1, \dots, l \right\}$:
\begin{align}
    & \begin{bmatrix}
    \text{---} & \beta_{\left\{ 1, \dots n \right\}}^{(1)} + \sum_{k=n+1}^l \xi_k^{(1)} \beta_{\left\{1, \dots, n \right\}}^{(k)} & \text{---} \\
    \vdots & \vdots & \vdots \\
    \text{---} & \beta_{\left\{ 1, \dots n \right\}}^{(n)} + \sum_{k=n+1}^l \xi_k^{(n)} \beta_{\left\{1, \dots, n \right\}}^{(k)} & \text{---}
    \end{bmatrix}
    \begin{bmatrix}
    \vrule \\ \ \\
    \mu_{i}^{\left( 1, \dots, n \right)} \\ \ \\
    \vrule
    \end{bmatrix} = \nonumber \\ & \nonumber \\ & = - 
    \begin{bmatrix}
    \beta_{i}^{(1)} + \sum_{k=n+1}^l \xi_k^{(1)} \beta_{i}^{(k)} \\ 
    \vdots \\
    \beta_{i}^{(n)} + \sum_{k=n+1}^l \xi_k^{(n)} \beta_{i}^{(k)}
    \end{bmatrix}. \label{eq:renormalization_group_linear_systems_appendix}
\end{align}
It now remains to find the renormalization group exponents for $a$. But assuming that we already know them for $b_1, \dots, b_n$, that is, we solved the linear systems above and found $\mu_j^{(i)}$ such that:
\begin{align}
    b_1^* &= B_{n+1}^{\mu_{n+1}^{(1)}} \dots B_l^{\mu_l^{(1)}} b_1; \ \ \ \ \ \  \dots \ \ \ \ \ b_n^* = B_{n+1}^{\mu_{n+1}^{(n)}} \dots B_l^{\mu_l^{(n)}} b_n.
\end{align}
We can look for exponents $\mu_{n+1}, \dots, \mu_l$ such that:
\begin{equation}\label{eq:qoi_renormalization_invariance}
    a^* = B_{n+1}^{\mu_{n+1}} \dots B_l^{\mu_l} a \ \ \Rightarrow \ \ \Pi^* \cdot \Pi_{n+1}^{*\xi_{n+1}} \cdots  \Pi_{l}^{*\xi_{l}} = \Pi \cdot \Pi_{n+1}^{\xi_{n+1}} \cdots  \Pi_{l}^{\xi_{l}}.
\end{equation}
To accomplish this, we will split the computations and begin by expressing $\Pi^*$ in terms of $B_{n+1}, \dots, B_l$ and $\Pi$:
\begin{align}
    & \Pi^* = a^{*\beta}b_1^{* \beta_1} \cdots b_l^{* \beta_l} a_1^{* \alpha_1} \cdots a_m^{* \alpha_m} =  \nonumber \\ & \nonumber \\ & \ \ \ \ \ \ = \left( B_{n+1}^{\mu_{n+1}} \cdots B_l^{\mu_l} a \right)^\beta \left( B_{n+1}^{\mu_{n+1}^{(1)}} \cdots B_{l}^{\mu_{l}^{(1)}} b_1 \right)^{\beta_1} \cdots \nonumber \\ & \nonumber \\ & \ \ \ \ \ \ \  \cdots \left( B_{n+1}^{\mu_{n+1}^{(n)}} \cdots B_{l}^{\mu_{l}^{(n)}} b_n \right)^{\beta_n} \left( B_{n+1} b_{n+1} \right)^{\beta_{n+1}} \cdots \nonumber \\ & \nonumber \\  & \ \ \ \ \ \  \cdots \left( B_{l} b_{l} \right)^{\beta_{l}} a_1^{\alpha_1} \cdots a_m^{\alpha_m} = \nonumber \\ & \nonumber \\ & B_{n+1}^{\beta \mu_{n+1} + \left\langle \beta_{\left\{ 1, \dots n \right\}} \ , \ \mu_{n+1}^{(1, \dots, n)} \right\rangle + \beta_{n+1}} \cdots B_{l}^{\beta \mu_l + \left\langle \beta_{\left\{ 1, \dots n \right\}} \ , \ \mu_{l}^{(1, \dots, n)} \right\rangle + \beta_l} \Pi.
\end{align}
By doing a similar calculation for $\Pi_k^{*\xi_k}$, we arrive at:
\small
\begin{align}
    \Pi_k^{*\xi_k} = B_{n+1}^{\xi_k \left( \left\langle \beta_{\left\{ 1, \dots, n \right\}}^{(k)} \ , \ \mu_{n+1}^{(1, \dots, n)} \right\rangle + \beta_{n+1}^{(k)} \right)} \cdots B_l^{\xi_k \left( \left\langle \beta_{\left\{ 1, \dots, n \right\}}^{(k)} \ , \ \mu_{l}^{(1, \dots, n)} \right\rangle + \beta_{l}^{(k)} \right)} \Pi_k^{\xi_k}.
\end{align}
\normalsize
By substituting $\Pi^*$ and $\Pi_k^{*\xi_k}$ for $k=n+1, \dots l$ in equation \eqref{eq:qoi_renormalization_invariance}, we get:
\small
\begin{align}
    & B_{n+1}^{\beta \mu_{n+1} + \left\langle \beta_{\left\{ 1, \dots, n\right\}} \ , \ \mu_{n+1}^{\left( 1, \dots\ n \right)} \right\rangle + \beta_{n+1} + \left\langle \sum_{k=n+1}^l \xi_k \beta_{\left\{ 1, \dots, n \right\}}^{(k)} \ , \ \mu_{n+1}^{(1, \dots, n)} \right\rangle + \sum_{k=n+1}^l \xi_k \beta_{n+1}^{(k)}} \cdots \nonumber \\ & \nonumber \\ & \cdots B_{l}^{\beta \mu_l + \left\langle \beta_{\left\{ 1, \dots, n\right\}} \ , \ \mu_l^{\left( 1, \dots\ n \right)} \right\rangle + \beta_l + \left\langle \sum_{k=n+1}^l \xi_k \beta_{\left\{ 1, \dots, n \right\}}^{(k)} \ , \ \mu_l^{(1, \dots, n)} \right\rangle + \sum_{k=n+1}^l \xi_k \beta_l^{(k)}} = 1.
\end{align}
\normalsize
By the same reasoning as before, we should have all the exponents above equal to zero and thus, for each $j = n+1, \dots, l$:
\scriptsize
\begin{equation}
    \mu_j = - \frac{1}{\beta} \left( \left\langle \beta_{\left\{ 1, \dots, n\right\}} \ , \ \mu_j^{\left( 1, \dots\ n \right)} \right\rangle + \left\langle \sum_{k=n+1}^l \xi_k \beta_{\left\{ 1, \dots, n \right\}}^{(k)} \ , \ \mu_j^{(1, \dots, n)} \right\rangle + \beta_j + \sum_{k=n+1}^l \xi_k \beta_j^{(k)}  \right)
\end{equation}
\normalsize

An analogous numerical remark is to be made here as well. we notice that the leftmost matrix in the linear systems of equation \eqref{eq:renormalization_group_linear_systems_appendix} does not depend on $i$. This means that we can proceed in a similar fashion to Buckingham's Similarity Group and transform the $l-n$ linear systems into a matrix equation. By defining:
\small
\begin{align}
        & \mathcal{A'} = \begin{bmatrix}
    \text{---} & \beta_{\left\{ 1, \dots n \right\}}^{(1)} + \sum_{k=n+1}^l \xi_k^{(1)} \beta_{\left\{1, \dots, n \right\}}^{(k)} & \text{---} \\
    \vdots & \vdots & \vdots \\
    \text{---} & \beta_{\left\{ 1, \dots n \right\}}^{(n)} + \sum_{k=n+1}^l \xi_k^{(n)} \beta_{\left\{1, \dots, n \right\}}^{(k)} & \text{---}
\end{bmatrix}; \\ & \nonumber \\ &
    \mu = \begin{bmatrix}
    \vrule & \cdots & \vrule \\ \ & \ & \ \\
    \mu_{n+1}^{\left( 1, \dots, n \right)} & \cdots & \mu_{l}^{\left( 1, \dots, n \right)} \\ \ & \ & \ \\
    \vrule & \cdots & \vrule
    \end{bmatrix}; \nonumber \\ \ & \nonumber \\ 
    & \mathcal{B'} = - \begin{bmatrix}
        \beta_{n+1}^{(1)} + \sum_{k=n+1}^l \xi_k^{(1)} \beta_{n+1}^{(k)} & \cdots & \beta_{l}^{(1)} + \sum_{k=n+1}^l \xi_k^{(1)} \beta_{l}^{(k)} \\
        \vdots & \ddots & \vdots \\
        \beta_{n+1}^{(n)} + \sum_{k=n+1}^l \xi_k^{(n)} \beta_{n+1}^{(k)} & \cdots & \beta_{l}^{(n)} + \sum_{k=n+1}^l \xi_k^{(n)} \beta_{l}^{(k)}, \nonumber
    \end{bmatrix},
\end{align}
\normalsize
and remembering that our main goal was to find the exponents in the matrix $\mu$, we can indeed look at it as a problem of solving the linear matrix equation:
\begin{equation} \label{eq:renormalization_group_linear_system}
    \mathcal{A'} \mu = \mathcal{B'},
\end{equation}
where $\mu$ is the unknown matrix. Unfortunately, there is no proof yet that the matrix $\mathcal{A'}$ is invertible if the constructions of the $\Pi_j$'s are independent by exponentiation and multiplication, but the authors believe it to be true.

\section{Buckingham's Similarity Group Invariance}
Throughout this paper, one of the main objectives was to find Buckingham's similarity group assuming that we know how the dimensionless quantities are constructed through the parameters of the phenomena at hand. It was also assumed that we know which of these parameters are dimensionally independent $(a_1, \dots, a_m)$ and which are dimensionally dependent $(b_1, \dots, b_l)$. 

It isn't hard to notice that, although we use such construction to find the exponents of the similarity group, the dimensionless quantities do not play any role when expressing it. Another curious fact is the lemma \ref{lemma:exponents_uniqueness}, which states that for every choice of vectors $\left\{ \beta^{(j)} \right\}_{j=1}^l \subseteq \R^l$, there is a unique choice of vectors $\left\{ \alpha^{(j)} \right\}_{j=1}^l \subseteq \R^m$ such that our construction is, in fact, dimensionless.

This means that the matrix $\mathcal{A}$ in equation \eqref{eq:buckingham_group_linear_system} depends solely on our choice of the matrix $\mathcal{B}$, so perhaps a more clear notation should be $\mathcal{A}\left( \mathcal{B} \right)$. Nevertheless, all those facts point to the following result:

\begin{nonameteo*}[\textbf{Claim III:}]
    Buckingham's similarity group must be independent of our choice of dimensionless construction.
\end{nonameteo*}
\begin{proof}
Mathematically, we are saying that there is an $l \times m$ matrix $\Delta^*$ such that $\Delta^*$ is the solution of:

\begin{equation}
    \mathcal{B} \Delta = \mathcal{A} \left( \mathcal{B} \right).
\end{equation}
for all possible choices of invertible matrices $\mathcal{B}$. In order to prove such theorem, we would need to investigate exactly how the matrix $\mathcal{A}$ depends on $\mathcal{B}$. In order to do this, remember that Buckingham's $\Pi$-Theorem tells us that there are exactly $m$ dimensions, say $D_1, \dots, D_m$ involved in our phenomena, and each of the dimensions of the parameters $a_1, \dots, a_m, b_1, \dots, b_l$ can be expressed in their terms through exponentiation and multiplication, i.e.:

\begin{align}
    &\left[ a_1 \right] = D_1^{\lambda_1^{(1)}} \dots D_m^{\lambda_m^{(1)}}; \ \ \ \ \dots \ \ \ \ ; \left[ a_m \right] = D_1^{\lambda_1^{(m)}} \dots D_m^{\lambda_m^{(m)}}; \nonumber \\ & \nonumber \\
    & \left[ b_1 \right] = D_1^{\gamma_1^{(1)}} \dots D_m^{\gamma_m^{(1)}}; \ \ \ \ \dots \ \ \ \ ; \left[ b_l \right] = D_1^{\gamma_1^{(l)}} \dots D_m^{\gamma_m^{(l)}}.
\end{align}
Notice that the dimension of each governing parameter can be represented as a vector of exponents in $\R^m$ so, for example, we can write the dimension of $a_1$ in the following way:
\begin{equation}
    \left[ a_1 \right] = D_1^{\lambda_1^{(1)}} \dots D_m^{\lambda_m^{(1)}} \sim \left( \lambda_1^{(1)}, \dots, \lambda_m^{(1)} \right)
\end{equation}
Let's start by defining the matrices:
\begin{equation}
        \Lambda = \begin{bmatrix}
        \lambda_1^{(1)} & \cdots & \lambda_1^{(m)} \\
        \vdots & \ddots & \vdots \\
        \lambda_m^{(1)} & \cdots & \lambda_m^{(m)}
    \end{bmatrix}; \ \ \ \ 
    \Gamma = \begin{bmatrix}
        \gamma_1^{(1)} & \cdots & \gamma_1^{(l)} \\
        \vdots & \ddots & \vdots \\
        \gamma_m^{(1)} & \cdots & \gamma_m^{(l)}
    \end{bmatrix}.
\end{equation}
Notice that, by reasoning similar to previous arguments, the fact that the parameters $a_1, \dots, a_m$ are dimensionally independent is equivalent to that that the vectors $\lambda^{(1)}, \dots, \lambda^{(m)}$ are linearly independent. We immediately conclude that the matrix $\Lambda$ is always invertible. Now let us suppose that we choose a vector $\beta^{(j)} \in \R^l$ in the construction of $\Pi_j$ as in the MDDP construction in equation \eqref{eq:mddp_construction_appendix}. In order to express the dimension of $b_1^{\beta_1^{(j)}} \dots b_l^{\beta_l^{(j)}}$, we perform the following calculation:

\begin{align}
    & \left[ b_1^{\beta_1^{(j)}} \dots b_l^{\beta_l^{(j)}} \right] = D_1^{\gamma_1^{(1)} \beta_1^{(j)}} \cdots D_m^{\gamma_m^{(1)} \beta_1^{(j)}} \cdots D_1^{\gamma_1^{(l)} \beta_l^{(j)}} \cdots D_m^{\gamma_m^{(l)} \beta_l^{(j)}} = \nonumber \\ & \nonumber \\ & = D_1^{\sum_{k=1}^l \gamma_1^{(k)} \beta_k^{(j)}} \cdots D_m^{\sum_{k=1}^l \gamma_m^{(k)} \beta_k^{(j)}} \sim \nonumber \\ & \nonumber \\
    & \sim \left( \sum_{k=1}^l \gamma_1^{(k)} \beta_k^{(j)}, \dots, \sum_{k=1}^l \gamma_m^{(k)} \beta_k^{(j)} \right) = \Gamma \beta^{(j)}.
\end{align}
With similar calculations, we know that if $\alpha \in \R^m$ is an exponent vector for $a_1, \dots, a_m$, the dimension of $a_1^{\alpha_1} \cdots a_m^{\alpha_m}$ will be $\Lambda \alpha$. But because $a_1,\dots, a_m$ is, in some sense, a basis for the dimension space, we can ask ourselves which exponents $\alpha^{(j)} \in \R^m$ should be chosen for the $a_i$'s to make $b_1^{\beta_1^{(j)}} \dots b_l^{\beta_l^{(j)}}$ dimensionless. Well, it is not hard to see that this choice must indeed be:
\begin{equation}
    \alpha^{(j)} = - \Lambda^{-1} \Gamma \beta^{(j)}.
\end{equation}
With this result, we can express the matrix $\mathcal{A}$ in our conjecture in terms of the matrix $\mathcal{B}$, with a little bit of matrix manipulation we arrive at the conclusion that:
\begin{equation}
    \mathcal{A} \left( \mathcal{B} \right) = \left( \Lambda^{-1} \Gamma \mathcal{B}^T \right)^T = \mathcal{B} \ \Gamma^T {\Lambda^{-1}}^T.
\end{equation}
So that when solving our previous matrix equation, we see that:
\begin{equation}
    \mathcal{B} \Delta = \mathcal{B} \ \Gamma^T {\Lambda^{-1}}^T \quad \Rightarrow \quad \Delta = \Gamma^T {\Lambda^{-1}}^T.
\end{equation}
This means that the solution to Buckingham's similarity group is in fact independent of our choice of dimensionless construction, and our claim is proven.
\end{proof}

\section*{References}

\nocite{*}
\bibliography{similarity}% Produces the bibliography via BibTeX.

\end{document}